\documentclass[10pt,twocolumn,letterpaper]{article}
\usepackage[T1]{fontenc}
\usepackage{amsmath,amsthm}
\usepackage{newtxtext,newtxmath}
\usepackage[letterpaper,margin=0.75in,columnsep=0.25in]{geometry}
\usepackage{graphicx,booktabs,multirow,makecell,tabularx}
\usepackage[table]{xcolor}
\usepackage{algorithm}
\usepackage[noend]{algpseudocode}
\usepackage[numbers,sort&compress]{natbib}
\usepackage{microtype}
\usepackage[font=small,labelfont=bf,skip=4pt]{caption}
\usepackage{enumitem}
\usepackage{url}
\usepackage[colorlinks,linkcolor=linkblue,citecolor=linkblue,urlcolor=linkblue]{hyperref}

\definecolor{linkblue}{HTML}{1A4E8A}
\definecolor{c1}{HTML}{0072B2}
\definecolor{head}{HTML}{EAF1F7}
\definecolor{tint1}{HTML}{E3EEF7}
\definecolor{band2}{HTML}{F7F7F7}

\newtheorem{proposition}{Proposition}
\newtheorem{corollary}[proposition]{Corollary}
\DeclareMathOperator{\logit}{logit}
\setlist{nosep,leftmargin=*}

\title{AdaSpark: Adaptive DSpark with Online Learning\\for Tree Verification and N-gram Fill}
\author{Liquan Liu\\ {\small\texttt{liquan@zeraix.com}} \and
  Yifan Zhang\\ {\small\texttt{fergus@zeraix.com}} \and
  Bowei Xu\\ {\small\texttt{xbw112233@gmail.com}}}
\date{\textit{Zeraix}\enspace$\cdot$\enspace 5 October 2026}

\begin{document}
\twocolumn[
\begin{@twocolumnfalse}
\maketitle
\begin{abstract}
Block drafters such as DSpark propose ranked candidates for several positions in one forward pass, and a tree verifier checks them in one pass of the target. The number of rows to verify trades the tokens a wider tree is expected to accept against the time a wider verify takes. Most schedulers that choose this number take the verify time from a table or model measured before serving, corrected online by at most one scale factor, and take acceptance from the drafter's confidence estimates or from a map fitted offline.

AdaSpark learns both quantities while it serves, with no profile, calibration or sweep in advance. It learns which verify widths are worth offering and fits each one's verify time as a function of context. It fits each candidate's acceptance probability to the target's verify outcomes, with the drafter's confidence head as one input, and orders and sizes the tree by that fit instead of by the head. The same model prices n-gram continuations of the request's own text, so drafted and text-derived candidates compete for rows in one best-first order. The width is chosen by pricing time at the long-run decode rate.

On single- and multi-turn conversations from six public datasets, on three dense targets and one mixture-of-experts target, AdaSpark decodes 1.5--3.1$\times$ faster than llama.cpp's DSpark with the same drafters. Our imparo engine with AdaSpark is 1.17--1.52$\times$ faster than imparo running with a three-token chain (the default llama.cpp setting); this gain comes from the scheduler alone. Without a width sweep, AdaSpark is never more than 0.3\% slower than the best pinned tree width on any dense target or context band. On the mixture-of-experts target it ties the best pinned width, and the other pinned widths from 4 to 16 rows are 5--14\% slower.

\smallskip
\noindent\textit{Keywords:} speculative decoding; tree verification; DSpark; online learning; verify cost; mixture of experts\\
\textit{Code and data:} \url{https://github.com/zeraix/imparo}
\end{abstract}
\vspace{1.5em}
\end{@twocolumnfalse}
]

\section{Introduction}\label{sec:1}

Speculative decoding \cite{leviathan2023fast,chen2023accelerating} drafts tokens cheaply and verifies them in one forward pass of the target model, keeping the longest prefix the target agrees with. Drafters that emit a \emph{tree} of candidates rather than a chain raise the tokens accepted per pass \cite{cai2024medusa,li2024eagle,miao2024specinfer,li2024eagleb,wang2025opt}. Block drafters such as DSpark \cite{cheng2026dspark} produce a whole block of positions in one forward pass, each with a ranked list of alternatives, so a tree built from one block can be as wide as the scheduler allows \cite{zhou2026treespark,li2026chains}. The scheduler must decide how many rows to verify.

That decision needs two quantities: the tokens a tree of $n$ rows is expected to accept, and the time a verify of $n$ rows takes. Many schedulers fix the budget by hand or by a sweep \cite{li2026chains,shen2026draft,liu2026talon}. Others learn a chain length online from measured outcomes, with no model of the verify's cost in width or context \cite{saremi2026adaptivesd,hou2025banditspec,sun2026delay}, or measure a few lengths in short trial phases \cite{saxena2025utility}. Those that price the verify read its time from a table, curve or model measured before serving: a profile taken at engine start \cite{cheng2026dspark,kwon2023efficient,pan2026making}, a table from an offline profiler \cite{zheng2024sglang}, a latency ladder measured once per deployment \cite{zhou2026treespark}, or a roofline model calibrated offline with at most one online correction factor \cite{oh2026bastion}. The two open-source serving engines that implement DSpark's scheduler each measure their table at one context length or one short range of contexts (\S\ref{sec:6}). Acceptance comes from the drafter's probabilities \cite{pan2026making,oh2026bastion}, its trained confidence head \cite{cheng2026dspark,huang2025specdec}, or a map fitted offline to the target's outcomes \cite{zhou2026treespark,wang2026taps}. Most of these schedulers already choose the size anew for every verify; they differ in where the cost and the acceptance come from (Table~\ref{tab:sizing}).

We measure the verify time on four targets (\S\ref{sec:3.2}). It is a step function of $n$, as reported for GPU tile allocation \cite{cheng2026dspark,oh2026bastion} and for llama.cpp's Metal backend \cite{chordiya2026lossless}, and the steps fall in different places for 8-bit and 4-bit weights. On a mixture-of-experts target it rises with every row, because each row is routed to its own experts \cite{pan2026making,saxena2025utility,huang2025moesd}. It also grows with context, faster for a wide verify than for a narrow one. The best fixed width therefore differs between targets (\S\ref{sec:3.4}). For acceptance, a model fitted to the target's own outcomes predicts better than the drafter's confidence head, as TreeSpark \cite{zhou2026treespark} found for a map fitted offline (\S\ref{sec:3.3}).

AdaSpark is a scheduler for DSpark tree verification that learns both quantities while it serves. It learns which verify widths to offer and fits each one's verify time as a function of context, fits an acceptance model to the target's outcomes, lets n-gram continuations of the request compete with the drafter's candidates for rows, and chooses each round's width from the two fitted models. It needs no cost profile, calibration or width sweep before serving. Its constants (Appendix~\ref{app:B}) were set once, on one target and a separate development corpus, and are used unchanged on all four.

\paragraph{Contributions.} 

\begin{enumerate}
\item An online verify-cost model that learns which widths to offer, measuring a new width only where it could change the choice. It prices each width as a function of context with a low-quantile loss, with no offline profile or calibration, and keeps a bank of samples per context band, so it fits every context it has seen although contexts arrive one request at a time (\S\ref{sec:3.2}).
\item An acceptance model fitted online to the target's outcomes: a multinomial over each parent's children that takes the drafter's confidence head as one input and replaces it in ordering and sizing the tree, with separate terms for n-gram candidates and for the evidence each source gives about the other (\S\ref{sec:3.3}).
\item N-gram candidates priced by the same model as the drafter's. Earlier methods that verify drafted and retrieved candidates in one tree divide the rows by fixed rules \cite{shen2026draft,quan2025rasd}; here the two sources compete for rows in one best-first order. Under the model, at a fixed width they cannot lower a verify's expected tokens, and they widen it only when the cost model says they pay for it (\S\ref{sec:3.3}).
\item A width rule for step-shaped, non-monotone costs that prices time at the long-run decode rate rather than at each round's own ratio, with the conditions under which it is optimal (\S\ref{sec:3.4}); on development prompts on LFM2.5-2.6B it ties the per-round ratio and widens trees at long context. Starting with nothing learned, it is never more than 0.3\% slower than the best pinned width on each dense target and ties it on the routed one (\S\ref{sec:3.4}, \S\ref{sec:4.5}).
\end{enumerate}

We implement AdaSpark in imparo, our inference engine, and test it against llama.cpp's DSpark with the same drafter files, on three dense targets and one mixture-of-experts target. The single- and multi-turn conversations are taken from six public datasets by a fixed rule; every reply runs to its end or to the 8,192-token cap, and every exclusion is counted. Appendix~\ref{app:A} gives the test set. The setup notes in the code repository (\texttt{dspark-paper/SETUP.md}) give the engine, the baseline, the targets, the machine, the arms, and how to reproduce every table and figure. An ablation on the two LFM2.5 targets adds the mechanisms in stages, and the gain over llama.cpp is split into the part the engine gives any algorithm and the part the scheduler adds (\S\ref{sec:4}).

\section{Background}\label{sec:2}

\subsection{DSpark drafters}\label{sec:2.1}

A DSpark drafter \cite{cheng2026dspark} is a small network that reads hidden states from selected layers of the target and proposes tokens for the $B$ positions after the last committed token, the \emph{anchor}. A parallel backbone produces logits for all $B$ positions in one forward pass; a lightweight sequential head (the \emph{Markov head}) then adds to each position's logits a low-rank bias from the token chosen at the position before. A \emph{confidence head} predicts, per position, the probability that the draft token survives verification given that the tokens before it did; in \cite{cheng2026dspark} it is trained on the acceptance rate of speculative sampling at temperature 1 and calibrated offline.

\subsection{Tree verification}\label{sec:2.2}

A verify round takes a tree of candidate tokens rooted at the anchor and runs the target once over all its nodes. Each node attends to the committed context and to its own ancestors. A verify of \emph{width} $n$ has $n$ rows: the anchor and $n - 1$ candidates. Under greedy decoding a candidate is accepted when its token equals the target's argmax at its parent's row. At most one child of an accepted node can match, so the accepted nodes form a single path from the anchor. The round commits that path and the target's argmax at the path's last node. A chain is the special case of a tree with one child per node.

At each position we keep the drafter's top $K = 8$ candidates: the top one is the drafter's \emph{pick}, the other $K - 1$ its \emph{siblings}. The candidates of one block form a tree of fixed shape (Figure~\ref{fig:shape}). The picks form a spine: the pick at position $i$ hangs from the pick at position $i -$~1, and the first from the anchor. The siblings at position $i$ hang from the same parent as the pick and are leaves: our engine runs the Markov head along the picks only, so it scores nothing below a sibling, and an accepted sibling ends the drafter's part of the round. PCTree \cite{li2026chains} and TreeSpark \cite{zhou2026treespark} give siblings children of their own by running the head once more for each expanded parent; we did not implement this. N-gram candidates (\S\ref{sec:3.3}) form chains that branch from the anchor and from each pick, and can run past the end of the drafter's block.

\begin{figure*}[t]
\centering
\includegraphics[width=0.8\textwidth]{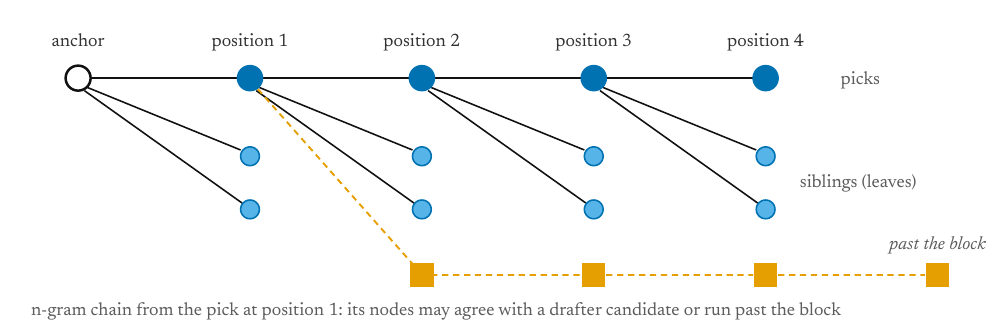}
\caption{Shape of the candidate tree, drawn for a block of $B = 4$ positions and $K = 3$ candidates per position. Dark blue nodes are the drafter's picks, light blue nodes its siblings, orange squares an n-gram chain. The targets here use $B = 7$ or 9 and $K = 8$; a verify takes the best-first prefix of these nodes (\S\ref{sec:3.1}).}
\label{fig:shape}
\end{figure*}

\section{AdaSpark}\label{sec:3}

\subsection{Overview}\label{sec:3.1}

Figure~\ref{fig:overview} shows one round and Algorithm~\ref{alg:round} states it. The drafter emits a block with its top-$K$ candidates per position, and n-gram chains from the request's own text are added to them. Each candidate gets $q$, the probability that it is accepted if its parent is. The candidates are ordered best-first by $Q$, the product of $q$ along the path from the anchor, so that every prefix of the order is itself a tree. The width chooser takes the prefix length that maximises Eq.~\eqref{eq:7}, the target verifies that prefix in one forward, and the round's measured time and accept outcomes update the two online models.

\begin{figure*}[t]
\centering
\includegraphics[width=0.8\textwidth]{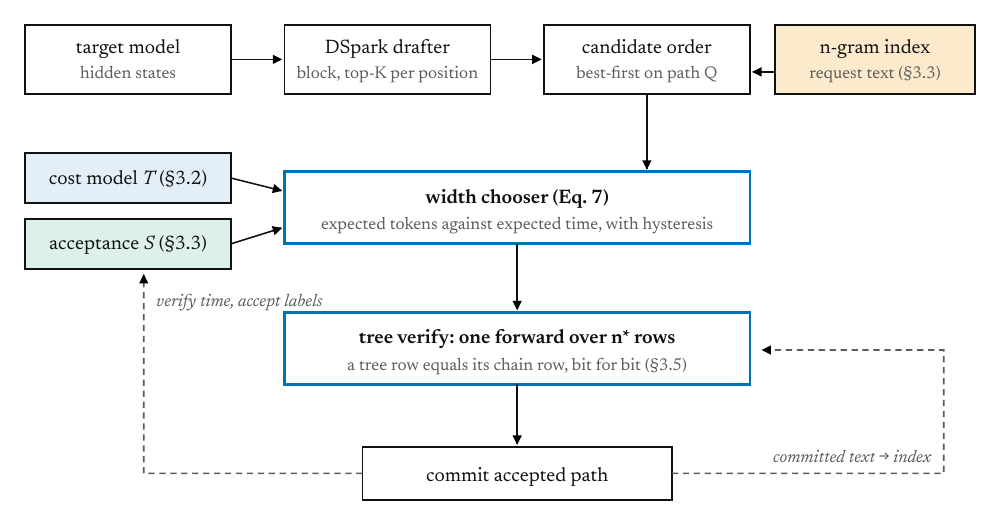}
\caption{One AdaSpark round. Solid arrows are the forward path; dashed arrows are the feedback that trains the two online models and extends the n-gram index. Both models persist across requests. The acceptance model's per-request intercept, and the per-request offset of the head-based estimate it competes with, restart with each request.}
\label{fig:overview}
\end{figure*}

\paragraph{One order serves every width.} A node can be accepted only if its parent is, and at most one child of an accepted node is accepted. The probability that candidate $j$ lies on the accepted path is therefore $Q_{j}$, and the expected number of accepted candidates of a tree is the sum of $Q$ over its nodes,

\begin{equation}\label{eq:1}
\begin{gathered}
S(n) = \sum_{j \in \text{first } n-1 \text{ nodes of } O} Q_j,\\
Q_j = \prod_{i \in \mathrm{path}(j)} q_i
\end{gathered}
\end{equation}

\begin{proposition}\label{prop:1}
Fix the round's $q$ for every candidate before the tree is chosen, and index the candidates so that every parent precedes its children. Order the candidates by decreasing $Q$, breaking ties toward the lower index. Then for every $n$, the first $n-1$ candidates of the order form a tree, and no tree of $n-1$ candidates has a larger $S$ under the model.
\end{proposition}

\begin{proof}
Since $q \leq$~1, a child's $Q$ is at most its parent's, and the parent has the lower index, so the tie rule places it first as well. Every candidate's parent therefore precedes it in the order, and every prefix is closed under taking parents, that is, a tree. The prefix consists of the $n - 1$ largest values of $Q$, so its sum is at least that of any other set of $n - 1$ candidates, trees included.
\end{proof}

This is the node-selection rule of EAGLE-2 \cite{li2024eagleb} and OPT-Tree \cite{wang2025opt}, and schedulers for block drafters and mixture-of-experts targets already use its prefix property to size trees \cite{pan2026making,oh2026bastion,zhou2026treespark}; we state it because the chooser relies on it at every width at once. The width decision reduces to choosing a prefix length, and one pass over the order gives $S(n)$ for every $n$. The engine computes each $q$ once per round over all of a parent's candidates, placed or not (Eq.~\eqref{eq:6}), as the proposition requires. Optimality holds under the model's $q$; \S\ref{sec:4.7} measures how well $q$ matches the target's outcomes.

\begin{algorithm*}[t]
\caption{One AdaSpark round. The widths offered to the chooser are the widths the cost model has learned (\S\ref{sec:3.2}). The head-based estimate is the drafter's confidence head with a per-request offset, and the guard decides which of it and the fitted model prices the tree (\S\ref{sec:3.3}); the incumbent is the width the chooser holds (\S\ref{sec:3.4}).}
\label{alg:round}
\small
\begin{algorithmic}[1]
\Require committed prefix $x$; drafter; target; cost model; acceptance model; n-gram index
\Ensure the tokens committed by this round
\State $\mathit{cands} \gets \mathrm{drafter.block}(\mathrm{target.hidden}(x))$ \Comment{$B$ positions, top-$K$ candidates each (\S\ref{sec:2.2})}
\State $\mathit{cands} \gets \mathit{cands} \cup \mathrm{ngram.chains}(x, \mathit{cands})$ \Comment{from the anchor and each pick (\S\ref{sec:3.3})}
\State $q_j \gets P(j \text{ accepted} \mid \text{parent accepted})$ \Comment{Eq.~\eqref{eq:6}, or the head-based estimate while the guard prefers it}
\State $O \gets \mathit{cands}$ by decreasing $Q_j = \prod q$ along $j$'s path \Comment{Proposition~\ref{prop:1}}
\If{a learned width has fewer than 3 measurements}
  \State $n^* \gets$ that width \Comment{cold start or a new split; the round still decodes}
\Else
  \State $n^* \gets \arg\max V(n)$, Eq.~\eqref{eq:7}, over the offered widths, kept at the incumbent unless beaten by the margin
\EndIf
\State $\mathit{tree} \gets$ the first $n^* - 1$ nodes of $O$, laid out depth-first
\State every 16th round with n-gram candidates: its last row goes to an unplaced n-gram node of the least-labelled bucket
\State $(\mathit{path}, t) \gets \mathrm{target.verify}(\mathit{tree})$ \Comment{one forward; $t$ is the measured verify time}
\State commit $\mathit{path}$ and the target's argmax after it
\State cost model (not in a request's first round): fold $t$ into $\mathrm{class}(n^*)$ at context $|x|$; fold the round's fixed cost $D$; split a class if a width inside could win
\State rate $\rho$ of the context band of $|x|$; gains $g_d, g_n$ from the rows above $n_0$, if $n^* > n_0$; the pair below $n_0$ from the rows under $\min(n^*, n_0)$
\State acceptance model: guard, then learn from every child of every accepted parent; request offset and intercept
\State $\mathrm{ngram.extend}(\mathit{path})$
\end{algorithmic}
\end{algorithm*}

\subsection{Online verify-cost model}\label{sec:3.2}

\paragraph{What a verify costs.} A verify of $n$ rows streams every weight of the target at least once, whatever $n$ is, and multiplies $n$ rows of activations by each weight matrix. The first cost is why speculation pays: on a dense target sixteen rows stream the same weights as one. The second is not linear in $n$. For a few rows the engine runs row kernels that read each weight once for all rows, up to a limit that depends on the weight format. Above it a matrix kernel computes whole tiles: at $n$ rows it runs one geometry $g(n)$ from a small set, chosen per device by the engine's tuner, pads $n$ up to that geometry's row count $\kappa(g)$, and may walk the weights more than once. Qualitatively,

\begin{equation}\label{eq:2}
T(n, c) \approx W\bigl(g(n)\bigr) + \gamma\,\kappa\bigl(g(n)\bigr) + \mathrm{Att}(n, c)
\end{equation}

where $W(g)$ is the time to stream the weights under geometry $g$, $\gamma$ the matrix time of one padded row, and $\mathrm{Att}$ the attention over $c$ tokens of context. $W$ and $\kappa$ are constant on intervals of $n$, so $T$ is nearly flat inside each interval and steps between them; we call each such interval a \emph{row class}. The steps need not be monotone: the tuner ranks geometries at the widths it measures, and a geometry fastest for a full tile can be slower for a padded one. Such steps have been reported for CUDA tile allocation \cite{cheng2026dspark,oh2026bastion} and for llama.cpp's Metal kernels on an Apple laptop of our machine's generation, where a 9-token verify was cheaper than a 7-token one because a matrix-matrix kernel replaced per-position matrix-vector products \cite{chordiya2026lossless}.

On a mixture-of-experts layer each row is routed to its own experts, so a verify of $n$ rows streams up to $n$ times the expert weights of one row. $W$ then depends on $n$ itself, and time rises inside a class as well as between classes. Earlier work models this cost by the number of distinct experts a verify touches \cite{pan2026making,saxena2025utility,huang2025moesd,mcdanel2026moe}. That number is bounded by the expert count (32 on our routed target, 4 per row), so expert traffic must level off once the rows together touch most experts. We did not separate expert traffic from the kernels' own steps; the verify time kept rising up to 48 rows, the widest we measured (Figure~\ref{fig:verify-cost}).

\paragraph{Measured verify time.} We ran DSpark decoding with the tree width pinned at each $n$ of 2, 4, 6, 8, 10, 12, 16, 20, 24, 32, 40 and 48 rows, at three contexts (443, 1,596 and 8,444 tokens) in two passes of opposite order, on prompts that continue public text (setup notes), and took the median verify time over the rounds that verified exactly $n$ rows, after the first two (Figure~\ref{fig:verify-cost}).

\begin{figure*}[t]
\centering
\includegraphics[width=0.8\textwidth]{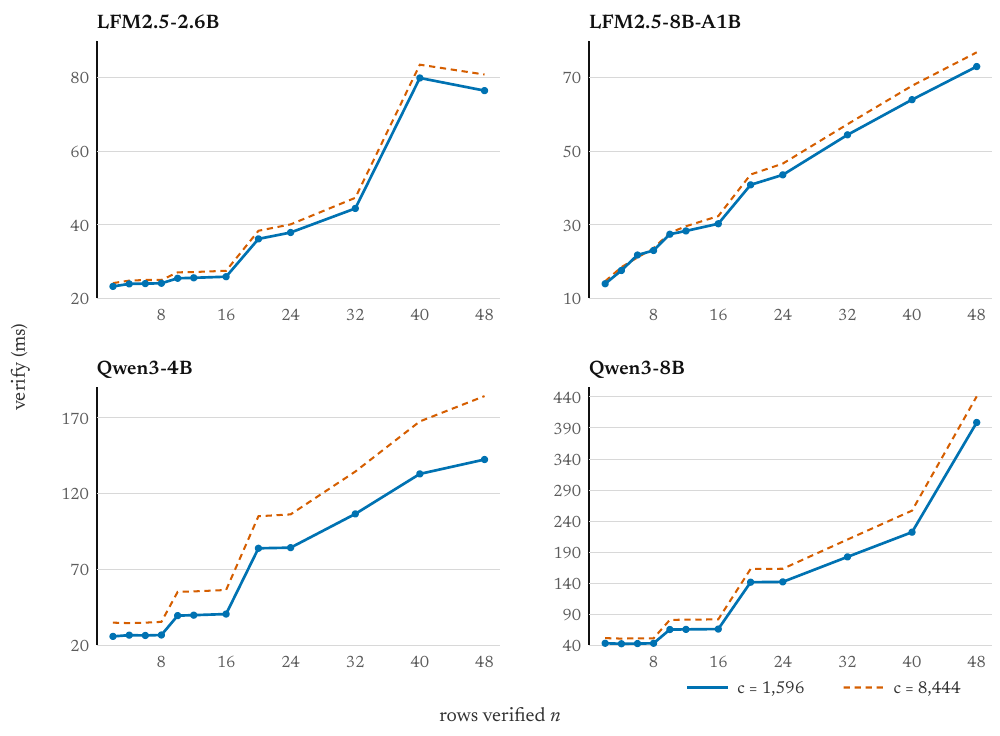}
\caption{Verify time against the number of rows verified, on each target, at two context lengths (median over two passes of the per-run median over rounds that verified exactly $n$ rows). Each panel has its own vertical scale. T(16)/T(2) at 1,596 tokens: LFM2.5-2.6B 1.11, LFM2.5-8B-A1B 2.16, Qwen3-4B 1.57, Qwen3-8B 1.53.}
\label{fig:verify-cost}
\end{figure*}

The dense targets show the steps of Eq.~\eqref{eq:2}, in different places. On LFM2.5-2.6B, whose weights are Q8\_0 (8-bit), a 16-row verify costs 11\% more than a 2-row one at a 1,596-token context, and a 20-row verify 40\% more than a 16-row one. Further out the curve is not monotone: a 48-row verify, which runs a different geometry, is 4.3\% cheaper than a 40-row one. On Qwen3-4B, whose weights are Q4\_K\_M (mostly 4-bit), the verify is flat to 8 rows and then steps steeply: at 1,596 tokens a 10-row verify costs 1.48 times an 8-row one, and a 20-row verify 2.07 times a 16-row one. Qwen3-8B, also Q4\_K\_M, has the same steps (1.51 and 2.14 times). The 4-bit matrix kernel processes rows in fragments of 8: up to 8 rows run in one fragment, close to the memory bandwidth, and from 9 rows in two fragments that reach about two thirds of that throughput. On the routed target time rises with every row, each routed to its own four experts: 8 rows cost 1.64 times as much as 2 rows, and 16 rows 2.16 times.

\paragraph{Context dependence.} A verify also slows as the conversation grows, because every verified row attends over every earlier token, and a wide verify slows faster. On Qwen3-4B, where every layer attends, a 16-row verify gains 2.4~ms per 1,000 tokens of context and an 8-row verify 1.3~ms, so the extra 8 rows cost 12~ms at 443 tokens and 21~ms at 8,444. On the two LFM2.5 targets most layers are short convolutions and only a few attend: the 8- and 16-row classes gain at most 0.2~ms per 1,000 tokens and the 32-row class under 0.45~ms. One constant per class is therefore not enough. Fitting

\begin{equation}\label{eq:3}
T_k(c) = \alpha_k + \beta_k\, c
\end{equation}

per class $k$ over the three contexts gives the intercepts and slopes of Table~\ref{tab:context-law}. The intercept $\alpha_{k}$ is the context-free part: the weight pass and the tile steps. The slope $\beta_{k}$ is the attention's time per token of context, and it grows with the class's row count, so a slope shared by all classes under-prices the wide classes at long context, which is where the ranking of widths changes.

\begin{table*}[t]
\centering\small
\caption{Verify time (ms) at three contexts for three row counts on each target, and the fit $T(c) = \alpha + \beta\cdot c$ of Eq.~\eqref{eq:3}.}
\label{tab:context-law}
\setlength{\tabcolsep}{4pt}
\begin{tabular}{lrrrrrrr}
\toprule
\rowcolor{head}\textbf{target} & \textbf{rows} & \textbf{$T$(443)} & \textbf{$T$(1,596)} & \textbf{$T$(8,444)} & \textbf{$\alpha$ (ms)} & \textbf{$\beta$ (ms per 1k tokens)} & \textbf{max residual} \\
\midrule
LFM2.5-2.6B & 8 & 24.0 & 24.1 & 25.0 & 23.9 & 0.13 & 0.1\% \\
LFM2.5-2.6B & 16 & 26.0 & 25.9 & 27.5 & 25.8 & 0.20 & 0.7\% \\
LFM2.5-2.6B & 32 & 43.9 & 44.4 & 47.3 & 43.7 & 0.43 & 0.0\% \\
LFM2.5-8B-A1B & 8 & 24.2 & 23.0 & 23.7 & 23.6 & $-$0.01 & 2.7\% \\
LFM2.5-8B-A1B & 16 & 32.3 & 30.3 & 32.4 & 31.3 & 0.11 & 3.9\% \\
LFM2.5-8B-A1B & 32 & 54.8 & 54.4 & 57.3 & 54.3 & 0.35 & 0.8\% \\
Qwen3-4B & 8 & 25.2 & 26.8 & 35.5 & 24.7 & 1.28 & 0.2\% \\
Qwen3-4B & 16 & 37.6 & 40.6 & 56.4 & 36.7 & 2.35 & 0.4\% \\
Qwen3-4B & 32 & 102.5 & 106.6 & 134.5 & 100.5 & 4.03 & 0.3\% \\
Qwen3-8B & 8 & 41.8 & 43.2 & 51.4 & 41.3 & 1.20 & 0.1\% \\
Qwen3-8B & 16 & 63.4 & 66.2 & 81.7 & 62.5 & 2.27 & 0.2\% \\
Qwen3-8B & 32 & 177.3 & 182.2 & 210.3 & 175.6 & 4.12 & 0.1\% \\
\bottomrule
\end{tabular}
\end{table*}

\paragraph{The model.} The cost model answers one query: the time a verify of class $k$'s offered width takes at context $c$. It is fitted from the verifies the engine runs, so measuring costs no extra work.

\paragraph{Which widths are offered.} The model learns the widths rather than taking them from the backend. It starts from a ladder of 2, 4, 8, 16, 32 and 64 rows and treats the widths between two neighbours as one \emph{class}, priced at its top width, the only width in it the chooser is offered. A class is thus an interval of this learned table; it coincides with a row class only where splits have found the kernel's steps. The model splits a class only where a width inside could change the choice. Between two measured widths $a < b$, a width inside is taken to cost at least $T(a)$ and to accept at most what $b - 1$ rows accept, which the round already knows; if even $b - 1$ rows at $T(a)$ cannot clear the chooser's switching margin against the width it holds, nothing inside is measured. Once the chooser has held a width for 8 rounds, the nearest class that passes this test is split, at the width next to the held width when the class's top costs more than 25\% above its bottom and at its midpoint otherwise, and the new width is measured three times. At most 24 widths are learned. A measurement round asking for more rows than any tree the drafter can offer teaches the table its reach (57 rows for the Qwen3 drafters). A step down inside a class, like the 48-row verify above, is found only if a split lands on it. The alternative, taking the classes from the backend's matrix kernels, fits only the 8-bit kernel on this engine (one class every 8 rows plus the tile boundaries of Eq.~\eqref{eq:2}); the 4-bit and routed targets would borrow a partition that is not theirs. Table~\ref{tab:classes} compares the two (\S\ref{sec:4.5}).

\begin{proposition}\label{prop:2}
Let $T$ be constant on each class, let $S_d$ and $S_{\mathrm{ng}}$ be non-decreasing in $n$, and let $g_d, g_n \geq 0$. Then some maximiser of $V$ in Eq.~\eqref{eq:7} (\S\ref{sec:3.4} defines its terms) is the top width of a class.
\end{proposition}

\begin{proof}
Inside a class $T$ does not change, and with non-negative gains the tokens term cannot fall as $n$ grows, so moving any width to the top of its class never lowers $V$.
\end{proof}

The gains are clamped to [0.25,~4] (\S\ref{sec:3.4}), so the condition always holds. When a wider class is cheaper than a narrower one, as at the 48-row inversion, each width is priced at the cheapest class whose top is at least that width, so the narrower, dearer class is never chosen. vLLM's running maximum over narrower classes (\S\ref{sec:6}) does the opposite: it raises the cheaper wide class to the narrower price. A stopping rule that grows the tree until the estimated speedup first falls \cite{oh2026bastion} assumes a convex cost and can stop below a cheaper wide class. SGLang's DSpark scheduler has the idea of Proposition~\ref{prop:2} as an option that fills a verify up to the CUDA-graph size its forward is padded to anyway \cite{zheng2024sglang}.

\paragraph{The context law.} Interference is one-sided: a busy GPU, a cold pipeline or a thermal limit makes a verify slower, and nothing makes it faster than the machine can run it. We therefore fit a low quantile, not the mean. For each class the model regresses the verify time $y_{t}$, normalised by the class's scale (the lowest of its first three measured times), on a basis in the context $u = c/8192$:

\begin{equation}\label{eq:4}
\begin{split}
&\min_{\theta}\, \sum_t \ell_\tau\bigl(y_t - \theta^{\top}\phi(u_t)\bigr) + \lambda_1\|\theta\|_1 + \lambda_2\|\theta\|_2^2,\\
&\phi(u) = \bigl[1, u, u^2, \sqrt{u}, \ln(1+u)\bigr],\\
&\ell_\tau(v) = v\,\bigl(\tau - \mathbf{1}[v<0]\bigr)
\end{split}
\end{equation}

\begin{proposition}\label{prop:3}
Let an observed verify time be $Y = t + \epsilon$, where $t$ is the machine's time for the verify and $\epsilon \geq 0$ is interference with $P(\epsilon = 0) = \pi$. For every $0 < \tau < \pi$ the $\tau$-quantile of $Y$ is $t$, while its mean is $t + E[\epsilon]$.
\end{proposition}

\begin{proof}
$P(Y < t) = 0 < \tau$ and $P(Y \leq t) = \pi > \tau$, so $t$ is the unique $\tau$-quantile.
\end{proof}

A least-squares fit absorbs every burst of interference into the price; a low quantile does not. We set $\tau = 0.15$ by hand; by Proposition~\ref{prop:3} it recovers the machine's time as long as more than 15\% of rounds run free of interference (a small two-sided jitter makes the fitted quantile slightly low). The quantile loss is Koenker and Bassett's \cite{koenker1978regression}. The context coefficients are non-negative, so the fitted time never falls as context grows. The basis is redundant on purpose: an affine law fits the dense targets, but other attention routes can add curvature past some context length. The penalty is an elastic net \cite{zou2005regularization} with $\lambda_{1} = 0.002$ and $\lambda_{2} = 0.01$ on the normalised time: over a short run the context barely moves and the basis terms are nearly collinear, and the L2 term keeps the weight shared instead of letting it jump between them. The solver is FTRL-Proximal \cite{mcmahan2013ad}, with the per-coordinate learning rate floored so the law can follow a machine that drifts. The fitted law is clamped to between 0.5 and 3 times the class's lowest observed time.

\paragraph{The sample bank.} The online pass alone does not fit the shape. Contexts arrive one request at a time, each sample pulls the intercept toward the current context, and only contrasts between requests teach the slope: on Qwen3-4B the online 8-row law priced 28.9, 29.8 and 35.0~ms at 443, 1,596 and 8,444 tokens, where the verify's 15th percentile was 25.2, 26.7 and 35.3~ms. Each law therefore keeps a bank of the newest 16 samples of every context band (the bit length of the context), stored with its weights and FTRL's accumulated squared gradients, so that a restored law learns at the rate it had reached. At the end of each request a batch refit trains a candidate law on the bank, bands interleaved, holding out the newest quarter of every band with four samples or more. The candidate replaces the online law if its quantile loss is lower on the judged samples: the held-out ones, plus every sample of a band too small to hold any out. On the same samples it prices 25.6, 26.7 and 35.0~ms. Sparse banks are common. A width run only to measure it holds three samples in a band, too few to hold a quarter out, and three online steps barely move a slope: on Qwen3-4B the 4-row law priced 26.1~ms at 8,444 tokens, where its own samples took 34.2~ms. The level (below) is not stored, so every long request started from that price and ran the width until its level corrected it. The refit therefore runs as many steps on a sparse bank as on a full one and judges it on all its samples; refitted, the 4-row law prices 34.2~ms. A misfit law does little harm to the running class, whose level corrects it within 8 rounds; it misleads the choice through the classes not running, which are the neighbours the chooser weighs against the running one.

\paragraph{The price.} The price the chooser reads for class $k$ is the product of three terms, each on its own timescale:

\begin{equation}\label{eq:5}
T(n_k, c) = \mathrm{law}_k(c)\cdot \mathrm{level}_k \cdot D_{\mathrm{now}}/D_k
\end{equation}

The law is the fit of Eq.~\eqref{eq:4}, a property of the kernels that changes slowly. The level is the minimum, over the class's last 8 verifies, of the measured time over the law's value, clamped to [0.5,~2]; it follows the machine within 8 rounds of the class running and, unlike an all-time minimum, rises again when the machine slows. The last factor corrects for staleness. Only the chosen class runs in a round, so every other class was last measured earlier, possibly on a quieter machine, and would look cheaper than it is. The round's fixed cost $D$ (the drafter's forward and the tree build) runs every round whatever the width, so it serves as a clock: each class records $D$ when it is measured, and its price is scaled by the current $D$ over that value. Until a class's law has enough samples, its price is its lowest measured time, scaled by the same clock. The first round of a request is not learned from. On all four targets it runs cold, its verify 5--11\% and its drafter 3--7\% slower than typical, while the second round is within 0.3\%; as the first sample of an empty level window it would price its width out of the offer, and that width would not run again to correct it. Pricing a width takes about 11~ns of host time and a training step, bank included, about 44~ns, so with 24 widths a round spends under 0.4~$\mu$s in the cost model; the refit at the end of a request takes about 0.04~ms per law.

\paragraph{Cold start.} A width with fewer than three measurements is measured by running rounds at it, narrowest first; these rounds still decode. The ladder's six widths take 18 rounds and each later split three more; a round whose tree comes out narrower than the width it measures does not count. From an empty store the first request spends 55 to 67 rounds measuring, 8--10\% of its decode time, and the rest of the instrumented run none (\S\ref{sec:4.6}). The table is stored, so a restart begins from the widths already measured.

\subsection{Acceptance model and n-gram candidates}\label{sec:3.3}

Two quantities available at draft time predict acceptance. The confidence head is a scalar from a small network trained offline for this purpose, calibrated on average over its training distribution. The pick's \emph{share} is its softmax mass among the drafter's top $K$ at that position, from the same forward's logits; it reports how peaked the drafter's distribution is, which separates a determined continuation, such as a closing bracket or the rest of an identifier, from an open one. Where the two disagree the share is the better guide: on LFM2.5-2.6B, over the development corpus's held-out prompts (setup notes), picks with head 0.34 and share 0.94 were accepted 75\% of the time, and picks with head 0.70 and share 0.43, 40\%. The model below uses both; Table~\ref{tab:accept} measures on every target how much better it predicts acceptance than the head-based estimate it replaces.

The head in \cite{cheng2026dspark} is trained on the acceptance rate of speculative sampling at temperature 1, while we decode greedily, accepting a candidate only if it equals the target's argmax; part of the head's loss may come from that mismatch. Fitting acceptance to the target's outcomes is not new: TreeSpark \cite{zhou2026treespark} fits a two-parameter map from the drafter's probability to acceptance on DSpark trees, and TAPS \cite{wang2026taps} trains a softmax over each parent's children, both offline. Our model is fitted online, per deployment, from an empty store, and carries terms for the n-gram source.

\begin{table*}[t]
\centering\small
\caption{Acceptance prediction on the evaluation subset: log loss per labelled node of the head-based estimate (the drafter's confidence head with its per-request offset) and of the fitted model, both scored on each round's labels before either learned from them, over every request after the first that did not reach the token cap. \emph{priced} is the share of rounds whose tree the fitted model priced; the last column gives rounds priced over tree rounds for the first request, which starts from an empty store.}
\label{tab:accept}
\setlength{\tabcolsep}{4pt}
\begin{tabular}{lrrrrrrrr}
\toprule
\rowcolor{head}\textbf{target} & \textbf{requests} & \textbf{nodes} & \textbf{head-based} & \textbf{fitted} & \textbf{change} & \textbf{requests won} & \textbf{priced} & \textbf{first request} \\
\midrule
LFM2.5-2.6B & 91 & 580,495 & 0.1410 & 0.1238 & $-$12.2\% & 91/91 & 100\% & 157/436 \\
LFM2.5-8B-A1B & 94 & 592,660 & 0.1520 & 0.1360 & $-$10.5\% & 94/94 & 100\% & 322/644 \\
Qwen3-4B & 87 & 567,698 & 0.1644 & 0.1400 & $-$14.8\% & 87/87 & 100\% & 779/838 \\
Qwen3-8B & 93 & 634,175 & 0.1572 & 0.1336 & $-$15.0\% & 93/93 & 100\% & 749/864 \\
\bottomrule
\end{tabular}
\end{table*}

\begin{figure*}[t]
\centering
\includegraphics[width=0.8\textwidth]{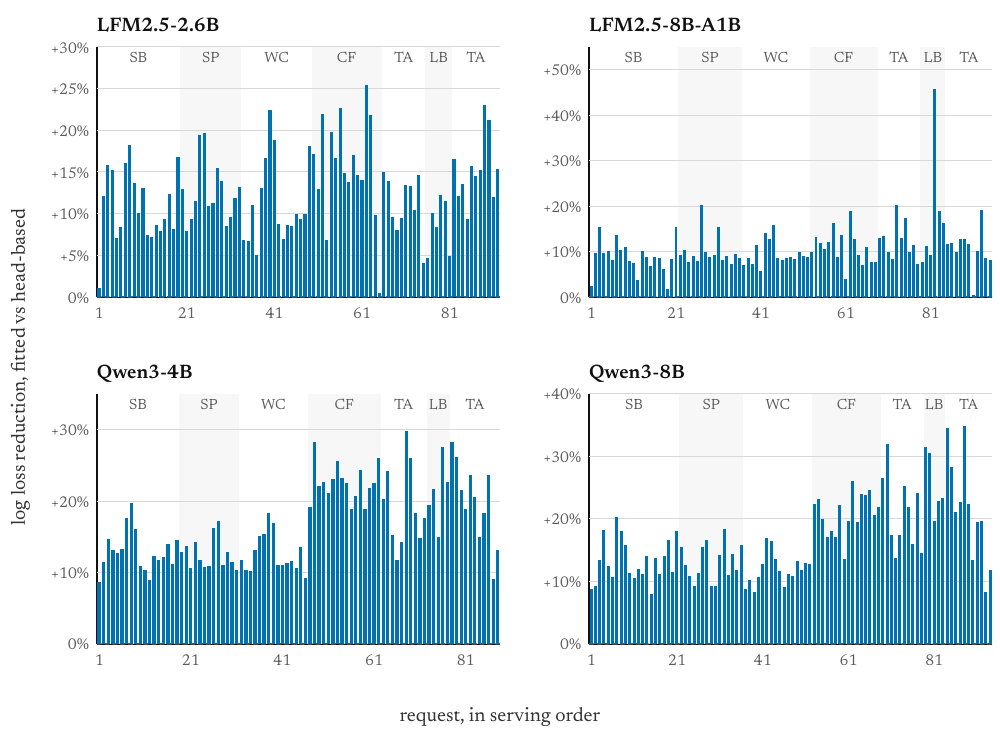}
\caption{The acceptance model while serving the evaluation subset from an empty store (looped replies left out): for each request, how much lower the fitted model's log loss per labelled node is than that of the head-based estimate it competes with (1 $-$ fitted / head-based), both scored before either learns from the request's labels. A bar above the dashed zero line (blue) is a request the fitted model predicted better. Shaded bands mark the six sources in serving order (SB Spec-Bench, SP SPEED-Bench, WC WildChat, CF ConvFinQA, TA ToolACE, LB LongBench). The fitted model is lower on 91 of 91, 94 of 94, 87 of 87, 93 of 93 requests after the first, in the order of the panels. At the 20 changes of source (five on each target), the first request of the new source ranks 0.47 on average among that source's other requests by this reduction (0 lowest, 1 highest) and below their median at 8 of them.}
\label{fig:accept-curve}
\end{figure*}

\paragraph{N-gram candidates.} AdaSpark adds \emph{n-gram fill} to the drafter's tree: continuations of the request's own text, prompt and reply, indexed by suffix. From the anchor and from each pick it follows a chain: the token that followed an earlier occurrence of the current context, read on along the earlier passage, or else the most frequent follower of the last three tokens. A chain step equal to a candidate the drafter already proposed marks that candidate as agreed; otherwise it adds an n-gram node. Chains stop 32 tokens below the anchor, past the drafter's block, or when nothing is proposed. The n-gram nodes enter the same best-first order as the drafter's, priced by the n-gram terms of Eq.~\eqref{eq:6}.

\paragraph{Form.} At most one child of an accepted parent can be the target's argmax, so the candidates of one parent and an explicit \emph{none} outcome share one softmax:

\begin{equation}\label{eq:6}
\begin{aligned}
q_j &= P(j\ \text{accepted} \mid \text{parent accepted})\\
    &= \frac{\exp(s_j)}{1 + \sum_{i \in \mathrm{cand}(\mathrm{parent})} \exp(s_i)},\\[2pt]
\text{pick: } s &= a_0 + a_1\logit h + a_2\logit\sigma + \psi + b_{\mathrm{req}},\\
\text{sibling: } s &= a_3 + a_4\log r + a_5\logit h\\
    &\quad + a_6\logit\sigma + \psi + b_{\mathrm{req}},\\
\text{n-gram: } s &= \logit p_{\mathrm{src}} + w_{\mathrm{bucket}} + \psi + b_{\mathrm{req}}
\end{aligned}
\end{equation}

Here $h$ is the confidence head at the position, $\sigma$ the pick's share, $r$ a sibling's share of the siblings' total, $\psi$ the context term (and, for drafter candidates, a depth term), $b_{\mathrm{req}}$ a per-request intercept, and $p_{\mathrm{src}}$ the n-gram source's measured top-1 precision on the committed text. A bucket indexes an n-gram candidate's evidence (the length of the copy it continues, or its follower count), its kind, and the number of n-gram nodes above it on its path. Two further terms connect the sources. A candidate both sources propose (\emph{agreed}) keeps its drafter score and adds $\max(0, \logit p_{\mathrm{src}})$ and the weight of an agreed bucket, indexed by the n-gram proposal's bucket and by whether the candidate is a pick or a sibling. A drafter candidate whose parent also has a different n-gram proposal (\emph{contradicted}) subtracts $\max(0, \logit p_{\mathrm{src}})$ and adds the weight of a contradiction bucket, indexed by that proposal's evidence and by pick or sibling. The sum in the denominator runs over every candidate of the parent, placed in the tree or not, so each $q$ is fixed for the round before the tree is chosen.

\paragraph{Why the n-gram needs its own terms.} An n-gram node has neither a head nor a share, and its acceptance depends on whether the text repeats, which the drafter's features do not see. With one intercept for both sources, the maximum-likelihood fit matches the average prediction to the average outcome over both sources together, so a source whose base rate differs from the pooled rate is mispriced, the two in opposite directions. A term per source makes each source's average hold separately. The sources also carry evidence about each other, which the agreed and contradicted terms price. Agreement multiplies an agreed node's odds by $P(\text{agree} \mid \text{correct}) / P(\text{agree} \mid \text{wrong}) = p/\bigl((1-p)c\bigr)$, where $c \leq 1$ is the chance that a wrong follower lands on this very candidate; $\logit p$ is a lower bound on its logarithm, and the floor at zero keeps a weak source from counting against a token it agrees with. Appendix~\ref{app:C} lists what the two terms corrected when they were added.

\paragraph{Learning.} With an empty store the drafter's seven coefficients start from a prior derived from the features' definitions, not from an offline fit: the pick scores $\ln 2 + \logit h$, a sibling scores $\log r$, and both share coefficients and every bucket weight start at zero. Each coefficient is updated online by FTRL-Proximal on its deviation from the prior and leaves the prior only once its accumulated gradient exceeds $3\sqrt{m}$ after $m$ updates; the request intercept adapts by AdaGrad \cite{duchi2011adaptive}. After each round every candidate of every accepted parent is labelled, placed or not, so a round also teaches the model about nodes it did not verify. The model competes with the \emph{head-based estimate} it replaces: a pick's confidence head, a sibling's share of the top-$K$ softmax mass and an n-gram node's source precision, all shifted by a per-request logit offset. A guard scores each round's labels under both before either learns from them, and the fitted model prices the tree only while its decayed log loss is the lower. The coefficients persist across requests; the offset and the intercept restart with each request.

\paragraph{Exploration.} Every drafter candidate of an accepted parent is labelled, so the drafter's terms need no exploration. An n-gram bucket can go unlabelled: its candidates hang below other n-gram nodes, and a bucket priced too low is never placed and so never corrected. Every 16th tree round that has n-gram candidates, the last row goes to the n-gram frontier node of the bucket with the fewest labels.

\begin{corollary}\label{prop:4}
At every width $n$, the first $n-1$ nodes of the merged order have at least the $S$ of any tree of $n-1$ drafter candidates, under the round's $q$, and the verify has the same width.
\end{corollary}

\begin{proof}
Every tree of drafter candidates is a tree of the merged set, so Proposition~\ref{prop:1} applied to the merged set bounds it.
\end{proof}

The comparison is under the round's $q$, which already counts the n-gram candidates in each parent's softmax. Given the model, adding n-gram candidates never makes a verify of a given width worse: they take rows only from drafter candidates they outrank, and the verify's time depends only on the width. They can also move the chooser to a wider class, when their expected tokens, scaled by their own gain $g_{n}$ (\S\ref{sec:3.4}), pay for the step; Table~\ref{tab:ngram} counts both cases on each target. Without such pricing the failure does occur: with a fixed split of a 60-node budget, Graft \cite{shen2026draft} found that retrieved candidates attached at the root displaced strong draft nodes and made decoding slower than the drafter's tree alone (3.62$\times$ against 3.73$\times$, averaged over five benchmarks); it avoids this by pruning the drafter's tree first and filling only the freed slots.

\subsection{Choosing the width}\label{sec:3.4}

For each offered width $n$ the chooser forms the round's expected tokens, $1 + S(n)$, and its expected time, $D + T(n, c)$. Prior cost-aware schedulers \cite{hong2026inference,pan2026making,oh2026bastion,liu2024turbospec}, DSpark's own scheduler \cite{cheng2026dspark} and the DSpark schedulers in vLLM and SGLang (\S\ref{sec:6}) maximise their ratio. AdaSpark instead maximises

\begin{equation}\label{eq:7}
\begin{split}
V(n) = {}& 1 + S(n_0) + g_d\,\Delta S_d(n) + g_n\,\Delta S_{\mathrm{ng}}(n)\\
         & - \rho_m\,\bigl(D + T(n, c)\bigr)
\end{split}
\end{equation}

where $n_{0}$ is a fixed reference width of 8 rows, so that learning a narrower width does not move the point the marginal is measured from, $\Delta S_{d}$ and $\Delta S_{\mathrm{ng}}$ are the expected tokens that the drafter's and the n-gram's nodes add above $n_{0}$ (negative below it: the tokens a narrower tree gives up), and $\rho_{m}$ is the long-run decode rate of the context band $m$ the round is in. It differs from the ratio in three ways.

\paragraph{(i) Time is priced at the long-run rate.} The quantity to optimise is the time to decode a reply, not the ratio of any one round.

\begin{proposition}\label{prop:5}
Split a reply's positions into context bands, band $m$ holding $N_m$ positions. Assume that within a band the round states are independent draws from a distribution that does not depend on the width rule, that the mean round time is positive, and that the reply's text does not depend on the rule. Then, as $N_m$ grows, a width rule minimises the band's expected decode time if and only if, for almost every state, it maximises $E[A - \rho_m^{*} C \mid \text{state}]$, where $A$ is a round's committed tokens, $C$ its time, and $\rho_m^{*}$ the band's best achievable rate.
\end{proposition}

\begin{proof}
By the renewal-reward theorem the long-run rate of a rule in band $m$ is $R_m = E[A]/E[C]$, and the band's decode time divided by $N_m$ tends to $1/R_m$. The text, and so each $N_m$, is fixed, so the reply's time is minimised by maximising each $R_m$. By Dinkelbach's theorem \cite{dinkelbach1967nonlinear}, $\max E[A]/E[C] = \rho^{*}$ exactly when $\max E[A - \rho^{*} C] = 0$. With the state distribution fixed, $E[A - \rho^{*} C]$ is the average over states of the conditional expectation, which is maximal if and only if the conditional expectation is maximal for almost every state.
\end{proof}

The independence assumption is an approximation: a round's width decides where the next round starts, and so what the drafter proposes next. The text assumption holds because decoding is greedy and, within the widths one accumulation order serves, a verify's rows do not depend on its width (\S\ref{sec:3.5}); the other rounds (wider verifies, up to 4\% of rounds, and the routed target's 2-row verifies, 13.4\%) can change the text only at a near tie (Table~\ref{tab:agree} counts how often the text differs). The per-round ratio is not this rule. Between a narrow and a wide width, the ratio takes the wide one when its marginal tokens per unit time exceed the round's own ratio; Proposition~\ref{prop:5} takes it when they exceed $\rho_{m}^{*}$. In a round better than its band the ratio refuses rows worth buying, and in a worse round it buys rows that are not.

AdaSpark uses the band's achieved rate in place of $\rho_{m}^{*}$: committed tokens over time, both decayed with a half-life of 64 rounds. This is a stochastic form of Dinkelbach's iteration: the rate the rule achieves becomes the price it uses next. The same renewal-reward argument has chosen a chain's draft length under network delay \cite{sun2026delay}; here it sets the tree width over a cost with steps. A band is the set of contexts with the same bit length (128--255 tokens, 256--511, and so on); until a band holds 8 rounds the chooser uses the per-round ratio. $D$ shifts every width's value in Eq.~\eqref{eq:7} equally; it enters the choice through $\rho_{m}$, measured over whole rounds, and through the ratio used before $\rho_{m}$ is known. On LFM2.5-2.6B, on 16 held-out development prompts (setup notes) in two interleaved passes, pricing at $\rho_{m}$ instead of the per-round ratio changed milliseconds per token by $-$0.39\% (standard error 0.46\%), a tie, and raised the mean width from 8.0--8.1 rows to 8.2--12.7 rows on the prompts with 4k--16k tokens of context, where the two rules disagree most.

\paragraph{(ii) The marginal is rescaled.} $S$, a sum of per-node probabilities, is a poor estimate of what extra rows buy. On LFM2.5-2.6B the realised gain of 16 rows over 8 was 1.69 times the predicted gain at 1,596 tokens and 0.885 times at 8,444: errors in opposite directions, which no single constant corrects. The gains $g_{d}$ and $g_{n}$, one per source, are therefore fitted as lines in context by ridge-regularised least squares of realised on predicted marginal, shrunk toward 1 and clamped to [0.25,~4]. The observations need no extra rounds: by Proposition~\ref{prop:1} and the width invariance of \S\ref{sec:3.5}, a round's accepted path also gives what each narrower prefix of its tree would have accepted. Above $n_{0}$ this needs a round wider than $n_{0}$, so only those rounds teach the gains, and the chooser's first incumbent is the widest offered width. Below $n_{0}$ every round teaches, through the rows between the narrowest offered width and the smaller of its own width and $n_{0}$; a width narrower than $n_{0}$ prices the rows it gives up with a second pair of gains, fitted the same way on those rows. One pair cannot serve both sides where the best width is $n_{0}$ itself: on Qwen3-4B at 8,444 tokens, 3 of 1,115 rounds ran wider than 8 rows and realised none of their predicted marginal, driving the fitted gain to its floor of 0.25, while the 450 rounds that revealed rows 2 to 8 realised 0.94 of theirs. Priced at the floor, a 4-row tree would seem to give up only a quarter of the tokens it does.

\paragraph{(iii) Hysteresis.} The chooser keeps its current width unless another beats it by a margin. Once $\rho_{m}$ is known, the alternative's $V$ must exceed the incumbent's by 0.04 times the incumbent's expected tokens; before that, the alternative's ratio must be 4\% higher. The margin was set on LFM2.5-2.6B as the smallest tested value at which the choice no longer flipped on noise; at about 0.05 and above it never left its incumbent.

\paragraph{No fixed width fits every target.} Table~\ref{tab:widths} compares AdaSpark's width with the same scheduler's width pinned, on the evaluation subset, whose replies run to their end (\S\ref{sec:4.1}), so one run spans contexts from a few hundred tokens to over ten thousand. A pinned arm keeps the rest of AdaSpark (the acceptance model, the n-gram candidates and a store that learns from empty), so the arms differ only in how the width is set. Every target is pinned at 8, 12 and 16 rows, and the routed target also at 4, 5 and 6, since there every row adds time; on the dense targets 2 to 8 rows take within 4\% of the 8-row time (Figure~\ref{fig:verify-cost}), so a narrower pinned width gives up tokens for little saving. The best pinned width is 16 rows on LFM2.5-2.6B and 8 on the other three targets, and one step off costs 20--23\% on the Qwen3 targets (12 rows); among pinned widths only a sweep finds the best one. AdaSpark needs no sweep: on the three dense targets its speed is 1.002, 0.999 and 1.005 times the best pinned width's over all replies, and at least 0.997 times in every context band. On the routed target the best pinned width, 8 rows, ties it (AdaSpark 1.003$\times$, interval 0.989--1.019), and 4, 5, 6, 12 and 16 rows are 4.7--14.3\% slower. There AdaSpark moves between widths from round to round, which no pinned width can do: in about the same time per round as a pinned 6-row tree (0.2\% less, over its two runs) it commits 2.8\% more tokens per round. Figure~\ref{fig:marginal} shows why the targets differ: going from 8 to 16 rows buys about 10\% more tokens per round on both LFM2.5 targets (medians 10.0\% and 10.4\%) but costs 6.8\% more time per round on the dense target and 27.9\% on the routed one.

\begin{table*}[t]
\centering\small
\caption{The width on the evaluation subset: AdaSpark against the same scheduler with its width pinned, every other part unchanged (both online models, the n-gram fill, a store learning from empty). The width columns are each pinned arm's decode speed over AdaSpark's, geometric mean of per-reply ratios (above 1 is faster than AdaSpark), with the best pinned width in bold, for all replies and by the context a reply was decoded at (prompt plus half the reply). AdaSpark runs before the pinned arms and again after them, and its speed on a reply is the mean of the two runs. The routed target is also pinned at 4, 5 and 6 rows, in a second stage run the same way and compared with its own two AdaSpark runs; a dash marks a width not run (\S\ref{sec:3.4}). \emph{AdaSpark / best} is AdaSpark over the best pinned width with its 95\% bootstrap interval over conversations; \emph{again} is AdaSpark's second run over its first. Looped and single-delta replies are left out as in Table~\ref{tab:e2e}.}
\label{tab:widths}
\setlength{\tabcolsep}{4pt}
\begin{tabular}{llrrrrrrrrr}
\toprule
\rowcolor{head}\textbf{target} & \textbf{replies} & \textbf{n} & \textbf{4} & \textbf{5} & \textbf{6} & \textbf{8} & \textbf{12} & \textbf{16} & \textbf{AdaSpark / best} & \textbf{again} \\
\midrule
LFM2.5-2.6B & all replies & 92 & -- & -- & -- & 0.907 & 0.968 & \textbf{0.998} & \makecell[r]{1.002$\times$\\ {\scriptsize [0.996, 1.008]}} & 0.995 \\
 & under 2,048 & 75 & -- & -- & -- & 0.909 & 0.970 & \textbf{1.000} & \makecell[r]{1.000$\times$\\ {\scriptsize [0.995, 1.007]}} & 0.995 \\
 & 2,048 and over & 17 & -- & -- & -- & 0.899 & 0.962 & \textbf{0.990} & \makecell[r]{1.010$\times$\\ {\scriptsize [0.991, 1.025]}} & 0.995 \\
LFM2.5-8B-A1B & all replies & 91 & 0.857 & 0.920 & 0.953 & \textbf{0.997} & 0.926 & 0.900 & \makecell[r]{1.003$\times$\\ {\scriptsize [0.989, 1.019]}} & 0.995 \\
 & under 2,048 & 74 & 0.847 & 0.910 & 0.947 & \textbf{0.996} & 0.930 & 0.902 & \makecell[r]{1.004$\times$\\ {\scriptsize [0.990, 1.020]}} & 1.002 \\
 & 2,048 and over & 17 & 0.899 & 0.969 & 0.980 & \textbf{0.998} & 0.911 & 0.891 & \makecell[r]{1.002$\times$\\ {\scriptsize [0.950, 1.042]}} & 0.964 \\
Qwen3-4B & all replies & 88 & -- & -- & -- & \textbf{1.001} & 0.801 & 0.818 & \makecell[r]{0.999$\times$\\ {\scriptsize [0.992, 1.007]}} & 1.014 \\
 & under 2,048 & 66 & -- & -- & -- & \textbf{1.003} & 0.796 & 0.811 & \makecell[r]{0.997$\times$\\ {\scriptsize [0.990, 1.005]}} & 1.017 \\
 & 2,048 and over & 22 & -- & -- & -- & \textbf{0.994} & 0.818 & 0.839 & \makecell[r]{1.006$\times$\\ {\scriptsize [0.991, 1.018]}} & 1.006 \\
Qwen3-8B & all replies & 94 & -- & -- & -- & \textbf{0.995} & 0.771 & 0.792 & \makecell[r]{1.005$\times$\\ {\scriptsize [1.000, 1.011]}} & 1.001 \\
 & under 2,048 & 74 & -- & -- & -- & \textbf{0.997} & 0.772 & 0.792 & \makecell[r]{1.003$\times$\\ {\scriptsize [0.999, 1.009]}} & 1.001 \\
 & 2,048 and over & 20 & -- & -- & -- & \textbf{0.989} & 0.768 & 0.793 & \makecell[r]{1.011$\times$\\ {\scriptsize [0.995, 1.028]}} & 0.999 \\
\bottomrule
\end{tabular}
\end{table*}

\begin{figure*}[t]
\centering
\includegraphics[width=0.8\textwidth]{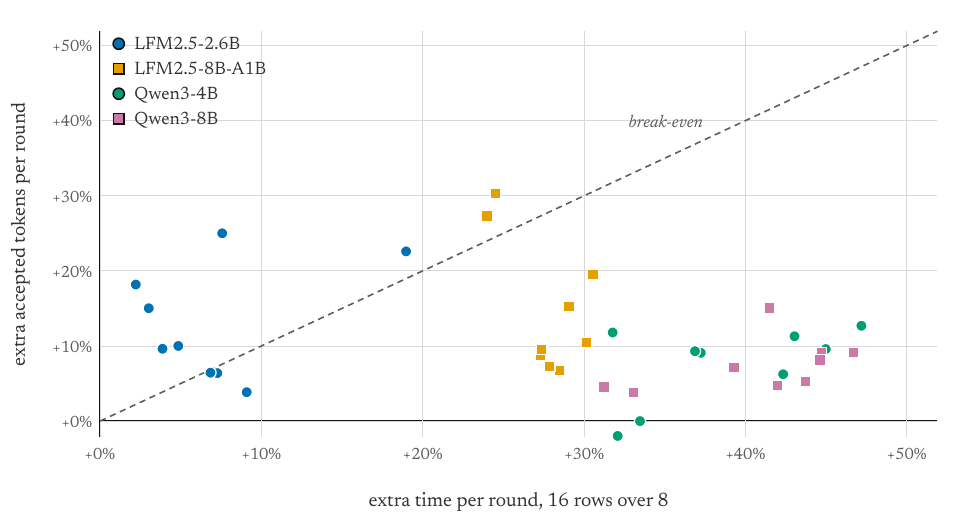}
\caption{What the second eight rows of a tree cost and what they buy. Each point is one target at one context (443, 1,596 or 8,444 tokens) and one of three prompts, from runs of 256 tokens at a fixed width of 8 and of 16 rows (acceptance model on, n-gram candidates off; setup notes): the extra time per round of a 16-row tree over an 8-row tree (the drafter included), against the extra tokens it accepts per round, both relative to the 8-row tree. Above the diagonal the wider tree decodes faster. The diagonal is the break-even line for every target.}
\label{fig:marginal}
\end{figure*}

\subsection{What the verify guarantees}\label{sec:3.5}

AdaSpark decodes greedily: every emitted token is the target's argmax on the verify row for exactly the emitted prefix. Two stronger properties were tested. A tree node whose ancestors are a chain's rows computes that chain row bit for bit, so a tree never changes a token a chain verify would have produced. And within one accumulation order of the matrix kernels, a verify's rows do not depend on its width: comparing every prefix of verifies of up to 64 rows with the same rows verified alone, on all four targets, rows are bit-identical among widths 3 to 16 on the three 4-bit targets and 3 to 24 on the 8-bit one. These ranges span several row classes; a wider verify accumulates in another order, so its rows differ by rounding, and 2 rows run a different kernel. The chooser's widths fell inside these ranges in 96.1\%, 85.6\%, 99.9\% and 99.4\% of rounds on LFM2.5-2.6B, LFM2.5-8B-A1B, Qwen3-4B and Qwen3-8B (Figure~\ref{fig:widths}); the rest were wider verifies and, on the routed target, 2-row verifies (13.4\% of its rounds). Equality with one-token decode does not hold: the verify's multi-row kernels accumulate in another order than the one-row kernel, so on a near tie their argmax can differ. Matching one-token decode was measured: it costs 2.2\% of decode speed on LFM2.5-2.6B and 5.0\% on Qwen3-4B if decode moves to the verify's kernels, or 25--70\% of verify time on 8-bit weights at 3 to 8 rows if the verify moves to decode's. We keep the fast kernels and report how often the text differs (\S\ref{sec:4.7}).

\section{Results}\label{sec:4}

\subsection{End to end on four targets}\label{sec:4.1}

Each arm serves the evaluation subset, 42 conversations (87 turns, 95 timed replies) taken by a fixed rule from six public datasets (Appendix~\ref{app:A}), one request at a time; every imparo arm starts from an empty learned store. The targets are LFM2.5-2.6B with 8-bit weights, LFM2.5-8B-A1B, a mixture of experts with 32 experts and 4 active per token \cite{liquidai2026lfm2}, and Qwen3-4B and Qwen3-8B \cite{yang2025qwen3}, the last three with mostly 4-bit weights, each with its published DSpark drafter: Liquid AI's for LFM2.5 \cite{liquidai2026lfm2} and DeepSeek's for Qwen3 \cite{deepseekai2026dspark}. Decoding is greedy with reasoning on, and every reply runs to its end or to 8,192 tokens; a reply that loops to that cap in either arm of a comparison is left out of that comparison and counted. A comparison is the geometric mean over replies of the ratio of two arms' decode speeds on the client's clock (generated tokens over the time from the first to the last streamed token), with a 95\% bootstrap interval over conversations. llama.cpp's DSpark runs at its defaults, a chain of at most three draft tokens, which we did not tune. AdaSpark runs twice on the evaluation subset (\S\ref{sec:3.4}), and its speed on a reply is the mean of the two runs. On the two ToolACE conversations whose tool schemas the set builder renamed (Appendix~\ref{app:A}), every arm's replies come from one separate run that starts from an empty store (setup notes). The mechanism measurements (Table~\ref{tab:ngram}, Figures~\ref{fig:round} to \ref{fig:calibration} and \S\ref{sec:4.6}) come from an instrumented run of AdaSpark with per-round logging on, over the \emph{short subset}: the first two turns of the first conversation of each source (11 timed replies). The logging costs time, so that run is not used for speed.

\begin{table*}[t]
\centering\small
\caption{End to end. Left: decode tok/s on the evaluation subset (Appendix~\ref{app:A}; generated tokens over decode time, summed over replies) and AdaSpark over llama.cpp DSpark (geometric mean of per-reply ratios, 95\% bootstrap interval over conversations). A reply that reached the 8,192-token cap is a repetition loop under greedy decoding: each tok/s column leaves out its own arm's looped replies and each ratio the replies that looped in either of its two arms; \emph{looped} counts the replies that looped in any arm. AdaSpark's speed on a reply is the mean of its two runs (\S\ref{sec:3.4}). A reply streamed as a single delta (a bare tool call) has no client-clock decode interval and is left out as well (none on LFM2.5-2.6B, 2 on LFM2.5-8B-A1B, none on Qwen3-4B, none on Qwen3-8B). Right: each engine's speculation speedup over its own autoregressive (AR) decode: its tok/s on the evaluation subset over its AR tok/s, which is measured on the short subset (setup notes) only (llama.cpp and imparo: LFM2.5-2.6B 43.2 and 46.1; LFM2.5-8B-A1B 99.0 and 105.0; Qwen3-4B 42.9 and 43.2; Qwen3-8B 25.6 and 25.8 tok/s). AR speed barely depends on the text; the short subset's contexts are shorter, where AR decode is faster, so these speedups are if anything understated.}
\label{tab:e2e}
\setlength{\tabcolsep}{4pt}
\begin{tabular}{lrrrrrrr}
\toprule
\rowcolor{head} & \multicolumn{5}{c}{\textbf{evaluation subset: tok/s and ratio}} & \multicolumn{2}{c}{\textbf{over own AR}} \\
\cmidrule(lr){2-6} \cmidrule(lr){7-8}
\rowcolor{head}\multirow{-2}{*}{\textbf{target}} & \makecell[r]{\bfseries llama.cpp\\ DSpark} & \makecell[r]{\bfseries imparo\\ DSpark} & \textbf{AdaSpark} & \makecell[r]{\bfseries AdaSpark /\\ llama.cpp} & \textbf{looped} & \makecell[r]{\bfseries llama.cpp\\ DSpark} & \textbf{AdaSpark} \\
\midrule
LFM2.5-2.6B & 66.9 & 83.8 & 126.3 & \makecell[r]{1.87$\times$\\ {\scriptsize [1.80, 1.95]}} & 3 & 1.55$\times$ & 2.74$\times$ \\
LFM2.5-8B-A1B & 79.1 & 105.5 & 122.1 & \makecell[r]{1.54$\times$\\ {\scriptsize [1.52, 1.57]}} & 2 & 0.80$\times$ & 1.16$\times$ \\
Qwen3-4B & 35.7 & 79.1 & 100.4 & \makecell[r]{2.83$\times$\\ {\scriptsize [2.73, 2.92]}} & 7 & 0.83$\times$ & 2.33$\times$ \\
Qwen3-8B & 21.7 & 50.9 & 68.2 & \makecell[r]{3.11$\times$\\ {\scriptsize [3.05, 3.18]}} & 1 & 0.85$\times$ & 2.64$\times$ \\
\bottomrule
\end{tabular}
\end{table*}

On the evaluation subset (Table~\ref{tab:e2e}, Figure~\ref{fig:e2e}) AdaSpark decodes 1.54--3.11$\times$ faster than llama.cpp's DSpark. Running llama.cpp's own configuration, a chain of at most three draft tokens, imparo is 1.23--2.37$\times$ faster than llama.cpp: this is the engine, with the algorithm held fixed. On the same engine and verify path (the chain is verified as a tree, \S\ref{sec:3.5}), AdaSpark is 1.17--1.52$\times$ faster than that configuration over the same replies: this is the scheduler. Speculation is worth much more on the dense LFM2.5 target than on the routed one: over each engine's own autoregressive decode, AdaSpark is 2.74$\times$ on the dense target and 1.16$\times$ on the routed one, where llama.cpp's DSpark is slower than llama.cpp's autoregressive decode (0.80$\times$). On both Qwen3 targets the two engines' autoregressive decode is at parity (imparo 1.01$\times$ llama.cpp), yet llama.cpp's DSpark is slower than its own autoregressive decode (0.83$\times$ and 0.85$\times$), while the same three-token chain in imparo is faster than imparo's (1.83$\times$ and 1.97$\times$); we have not analysed llama.cpp's implementation to explain the difference.

\begin{figure*}[t]
\centering
\includegraphics[width=0.8\textwidth]{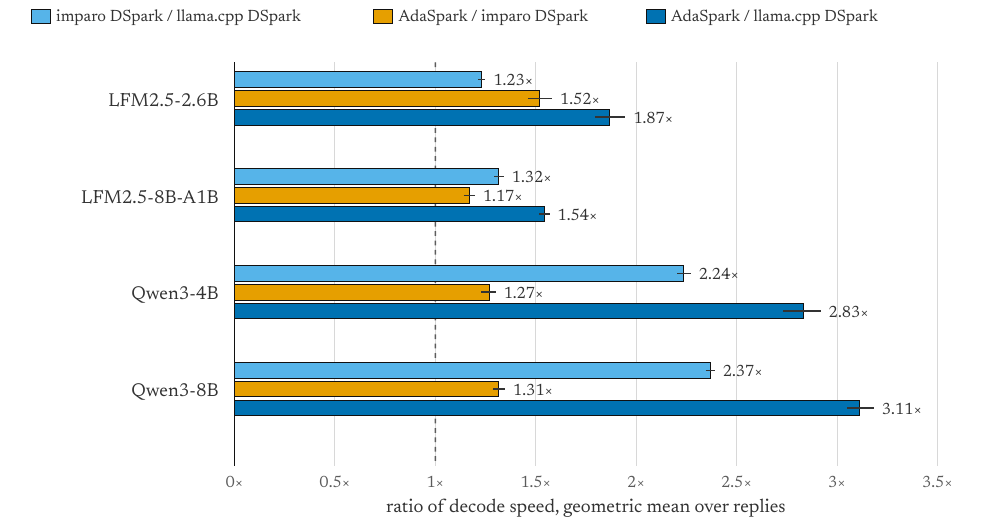}
\caption{Where the gain over llama.cpp comes from. The engine factor is imparo DSpark over llama.cpp DSpark, the two engines running the same algorithm (a chain of at most three draft tokens); the scheduler factor is AdaSpark over imparo DSpark, on the same engine; the total is AdaSpark over llama.cpp DSpark, over the replies both engines served; it is close to the product of the two, whose reply sets differ only by replies that looped in one arm. Whiskers are 95\% bootstrap intervals over conversations.}
\label{fig:e2e}
\end{figure*}

\begin{figure*}[t]
\centering
\includegraphics[width=0.8\textwidth]{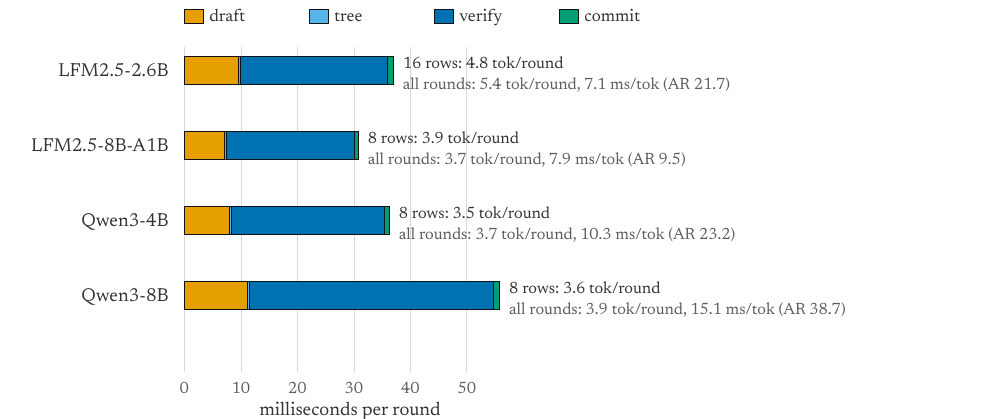}
\caption{Anatomy of an AdaSpark round on each target, at the width its rounds most often ran: the drafter's forward, the tree build, the verify and the commit (medians), and the tokens committed per round at that width. The text beside each bar also gives the tokens per round and milliseconds per token over all rounds, other widths included, beside the same engine's autoregressive (AR) milliseconds per token on the short subset. LFM2.5-2.6B: draft 9.6, verify 26.1 ms at 16 rows (90\% of rounds); LFM2.5-8B-A1B: draft 7.1, verify 22.8 ms at 8 rows (34\% of rounds); Qwen3-4B: draft 8.0, verify 27.1 ms at 8 rows (97\% of rounds); Qwen3-8B: draft 11.2, verify 43.2 ms at 8 rows (97\% of rounds).}
\label{fig:round}
\end{figure*}

Figure~\ref{fig:round} splits a round into its parts. A token costs 7.1~ms against 21.7~ms in autoregressive decode on LFM2.5-2.6B, and 10.3 and 15.1~ms against 23.2 and 38.7~ms on the Qwen3 targets. On the routed target it costs 7.9~ms against 9.5~ms: only about a billion of its eight billion parameters are read per token, so its autoregressive decode is fast, the drafter's fixed cost is a larger share of a round, and there is less for speculation to win.

\paragraph{Held-out conversations.} The 31 conversations (61 turns) of the test set outside the evaluation subset had not been run by any arm when the subset was measured. llama.cpp's DSpark and AdaSpark then ran on them, AdaSpark again from an empty store (Table~\ref{tab:heldout}). AdaSpark's lead there is within 4\% of its lead on the evaluation subset on every target, with overlapping intervals, and over the whole test set it is 1.52--3.10$\times$.

\begin{table*}[t]
\centering\small
\caption{AdaSpark over llama.cpp DSpark on the evaluation subset, on the held-out conversations (the 31 conversations, 61 turns, of the test set outside the subset, which no arm had run when the subset was measured) and on the whole test set (73 conversations, 148 turns): geometric mean of per-reply decode ratios with the 95\% bootstrap interval over conversations, and the replies it rests on. Looped and single-delta replies are left out as in Table~\ref{tab:e2e}; the last column counts the held-out replies that reached the cap.}
\label{tab:heldout}
\setlength{\tabcolsep}{4pt}
\begin{tabular}{lrrrrrrr}
\toprule
\rowcolor{head} & \multicolumn{2}{c}{\textbf{evaluation subset}} & \multicolumn{2}{c}{\textbf{held out}} & \multicolumn{2}{c}{\textbf{whole test set}} &  \\
\cmidrule(lr){2-3} \cmidrule(lr){4-5} \cmidrule(lr){6-7}
\rowcolor{head}\multirow{-2}{*}{\textbf{target}} & \textbf{ratio} & \textbf{replies} & \textbf{ratio} & \textbf{replies} & \textbf{ratio} & \textbf{replies} & \multirow{-2}{*}{\makecell[r]{\bfseries looped,\\ held out}} \\
\midrule
LFM2.5-2.6B & \makecell[r]{1.87$\times$\\ {\scriptsize [1.80, 1.95]}} & 92 & \makecell[r]{1.80$\times$\\ {\scriptsize [1.73, 1.89]}} & 64 & \makecell[r]{1.84$\times$\\ {\scriptsize [1.79, 1.90]}} & 156 & 3 \\
LFM2.5-8B-A1B & \makecell[r]{1.54$\times$\\ {\scriptsize [1.52, 1.57]}} & 93 & \makecell[r]{1.50$\times$\\ {\scriptsize [1.46, 1.54]}} & 65 & \makecell[r]{1.52$\times$\\ {\scriptsize [1.50, 1.55]}} & 158 & 1 \\
Qwen3-4B & \makecell[r]{2.83$\times$\\ {\scriptsize [2.73, 2.92]}} & 88 & \makecell[r]{2.88$\times$\\ {\scriptsize [2.83, 2.94]}} & 65 & \makecell[r]{2.85$\times$\\ {\scriptsize [2.79, 2.91]}} & 153 & 2 \\
Qwen3-8B & \makecell[r]{3.11$\times$\\ {\scriptsize [3.05, 3.18]}} & 94 & \makecell[r]{3.07$\times$\\ {\scriptsize [3.00, 3.14]}} & 67 & \makecell[r]{3.10$\times$\\ {\scriptsize [3.05, 3.14]}} & 161 & 0 \\
\bottomrule
\end{tabular}
\end{table*}

\subsection{Ablation}\label{sec:4.2}

\begin{table*}[t]
\centering\small
\caption{Ablation on the dense and the routed LFM2.5 targets. The first row of each block is llama.cpp DSpark's tok/s per source; every other cell is the arm over it, the geometric mean of per-reply decode ratios (looped replies excluded), with the 95\% bootstrap interval over conversations in the last column. From the fixed tree down, each row adds to the row above: the cost model's width, then the fitted acceptance model and n-gram candidates together. Sources: SB Spec-Bench, SP SPEED-Bench, WC WildChat, CF ConvFinQA, TA ToolACE, LB LongBench.}
\label{tab:ladder}
\setlength{\tabcolsep}{4pt}
\begin{tabular}{lrrrrrrr}
\toprule
\rowcolor{head}\textbf{arm} & \textbf{SB} & \textbf{SP} & \textbf{WC} & \textbf{CF} & \textbf{TA} & \textbf{LB} & \textbf{all replies} \\
\midrule
\rowcolor{band2}\multicolumn{8}{c}{\textit{LFM2.5-2.6B}} \\
llama.cpp DSpark (tok/s) & 67.2 & 66.6 & 60.5 & 72.9 & 64.9 & 71.8 & 66.9 \\
imparo DSpark (3-token chain) & 1.24 & 1.23 & 1.27 & 1.21 & 1.20 & 1.26 & \makecell[r]{1.23$\times$\\ {\scriptsize [1.21, 1.25]}} \\
full-block chain & 1.62 & 1.59 & 1.60 & 1.76 & 1.61 & 1.71 & \makecell[r]{1.64$\times$\\ {\scriptsize [1.59, 1.68]}} \\
fixed 16-row tree & 1.85 & 1.78 & 1.79 & 1.91 & 1.78 & 1.91 & \makecell[r]{1.82$\times$\\ {\scriptsize [1.78, 1.87]}} \\
cost-model width & 1.83 & 1.77 & 1.79 & 1.91 & 1.72 & 1.92 & \makecell[r]{1.81$\times$\\ {\scriptsize [1.76, 1.86]}} \\
\rowcolor{tint1}\textbf{AdaSpark (+ acceptance model, n-gram)} & \textbf{1.87} & \textbf{1.86} & \textbf{1.82} & \textbf{2.07} & \textbf{1.69} & \textbf{2.18} & \textbf{\makecell[r]{1.87$\times$\\ {\scriptsize [1.80, 1.95]}}} \\
\rowcolor{band2}\multicolumn{8}{c}{\textit{LFM2.5-8B-A1B}} \\
llama.cpp DSpark (tok/s) & 79.4 & 80.8 & 70.7 & 90.4 & 78.3 & 82.0 & 79.1 \\
imparo DSpark (3-token chain) & 1.32 & 1.31 & 1.39 & 1.27 & 1.30 & 1.36 & \makecell[r]{1.32$\times$\\ {\scriptsize [1.29, 1.34]}} \\
full-block chain & 1.26 & 1.24 & 1.28 & 1.32 & 1.29 & 1.23 & \makecell[r]{1.28$\times$\\ {\scriptsize [1.24, 1.31]}} \\
fixed 16-row tree & 1.33 & 1.30 & 1.39 & 1.39 & 1.37 & 1.31 & \makecell[r]{1.35$\times$\\ {\scriptsize [1.32, 1.38]}} \\
cost-model width & 1.51 & 1.48 & 1.54 & 1.54 & 1.50 & 1.53 & \makecell[r]{1.51$\times$\\ {\scriptsize [1.49, 1.54]}} \\
\rowcolor{tint1}\textbf{AdaSpark (+ acceptance model, n-gram)} & \textbf{1.50} & \textbf{1.55} & \textbf{1.55} & \textbf{1.57} & \textbf{1.55} & \textbf{1.60} & \textbf{\makecell[r]{1.54$\times$\\ {\scriptsize [1.52, 1.57]}}} \\
\bottomrule
\end{tabular}
\end{table*}

On the dense target (Table~\ref{tab:ladder}), verifying the drafter's whole block as a chain instead of three picks is worth 1.33$\times$, a 16-row tree another 1.11$\times$, and the cost model's width is within 1\% of the fixed tree (0.991$\times$, interval 0.981--1.000). The fixed arm uses 16 rows because that is this target's best pinned width (Table~\ref{tab:widths}); the cost model chose that width on its own in 59\% of its rounds on the instrumented run, and 14 or 15 rows in another 26\%. The fitted acceptance model and the n-gram candidates, added together because the n-gram candidates are priced by the model's n-gram terms (\S\ref{sec:3.3}), are worth 3.4\% (interval 1.4--5.5\%): 13.4\% on LongBench and 8.4\% on ConvFinQA, where the reply repeats its input. On the routed target the order of the gains changes. The full-block chain is no faster than the three-pick chain (0.969$\times$, interval 0.936--1.000), because each extra row costs expert traffic, and a fixed 16-row tree recovers little. The cost model's width, which spread its rounds over 2 to 16 rows (8 rows in 48\% of them, 4 rows in 17\%, 2 rows in 8\%), is 11.8\% faster than the fixed 16-row tree, which is not this target's best width (against the best, 8 rows, AdaSpark ties; Table~\ref{tab:widths}). The acceptance model and the n-gram candidates add 1.8\% there (interval 0.5--3.4\%).

\subsection{What n-gram fill adds}\label{sec:4.3}

\begin{table*}[t]
\centering\small
\caption{What n-gram fill contributes, from the instrumented AdaSpark run on the short subset. Rows are the verified nodes past each round's anchor; accepted tokens are the draft tokens committed. A round is widened when its n-gram candidates move the chooser to a wider class than it takes with them priced at zero.}
\label{tab:ngram}
\setlength{\tabcolsep}{4pt}
\begin{tabular}{lrrrr}
\toprule
\rowcolor{head}\textbf{} & \textbf{LFM2.5-2.6B} & \textbf{LFM2.5-8B-A1B} & \textbf{Qwen3-4B} & \textbf{Qwen3-8B} \\
\midrule
verified rows held by n-gram-only nodes & 15.2\% & 9.4\% & 7.2\% & 8.1\% \\
accepted draft tokens from n-gram-only nodes & 15.8\% & 7.8\% & 5.1\% & 7.8\% \\
accepted draft tokens from nodes both sources proposed & 18.1\% & 18.5\% & 22.1\% & 22.4\% \\
acceptance rate, drafter-only nodes & 24\% & 46\% & 33\% & 34\% \\
acceptance rate, nodes both proposed & 70\% & 72\% & 64\% & 70\% \\
acceptance rate, n-gram-only nodes & 30\% & 40\% & 26\% & 37\% \\
rounds with n-gram candidates & 906 & 2,012 & 2,253 & 2,650 \\
of those, rounds the n-gram widened the verify & 10.9\% & 7.5\% & 2.9\% & 4.2\% \\
\bottomrule
\end{tabular}
\end{table*}

On LFM2.5-2.6B (Table~\ref{tab:ngram}), n-gram-only nodes hold 15.2\% of the verified rows and supply 15.8\% of the accepted draft tokens, and nodes both sources proposed supply another 18.1\%. Agreement is strong evidence: an agreed node is accepted 70\% of the time, against 24\% for a drafter-only node. N-gram-only nodes are accepted 30\% of the time; their value is that they continue past the drafter's block when the text repeats. In 10.9\% of the rounds with n-gram candidates they moved the chooser to a wider class, and 65 of the 83 verifies of 24 rows or more were such widenings. On the routed target the n-gram-only shares are smaller and fewer rounds widen, because every extra row costs expert traffic. By source, the n-gram's share of accepted tokens (n-gram-only and agreed) runs from 7.7\% on WildChat to 47.3\% on LongBench on LFM2.5-2.6B, and from 10.5\% to 39.8\% on the routed target. On the Qwen3 targets the n-gram contributes mostly through agreement: agreed nodes supply over a fifth of the accepted draft tokens. One Qwen3-4B reply of this run looped to the 8,192-token cap and is left out; inside it, n-gram-only nodes were accepted 86\% of the time, which is why loops are left out of every statistic.

\subsection{By workload}\label{sec:4.4}

\begin{table}[t]
\centering\small
\caption{The scheduler's gain by workload: AdaSpark over imparo DSpark (same engine, a chain of at most three draft tokens), geometric mean of per-reply decode ratios per source, looped replies excluded. Darker cells are larger gains.}
\label{tab:sources}
\setlength{\tabcolsep}{4pt}
\begin{tabular}{lrrrr}
\toprule
\rowcolor{head}\textbf{source} & \makecell[r]{\bfseries LFM2.5-\\2.6B} & \makecell[r]{\bfseries LFM2.5-\\8B-A1B} & \textbf{Qwen3-4B} & \textbf{Qwen3-8B} \\
\midrule
Spec-Bench & \cellcolor{c1!27}1.51 & \cellcolor{c1!7}1.14 & \cellcolor{c1!14}1.28 & \cellcolor{c1!17}1.33 \\
SPEED-Bench & \cellcolor{c1!27}1.51 & \cellcolor{c1!9}1.19 & \cellcolor{c1!12}1.24 & \cellcolor{c1!16}1.31 \\
WildChat & \cellcolor{c1!23}1.43 & \cellcolor{c1!6}1.12 & \cellcolor{c1!12}1.24 & \cellcolor{c1!17}1.32 \\
ConvFinQA & \cellcolor{c1!38}1.71 & \cellcolor{c1!12}1.24 & \cellcolor{c1!19}1.36 & \cellcolor{c1!21}1.41 \\
ToolACE & \cellcolor{c1!22}1.41 & \cellcolor{c1!9}1.19 & \cellcolor{c1!12}1.23 & \cellcolor{c1!12}1.24 \\
LongBench & \cellcolor{c1!40}1.73 & \cellcolor{c1!6}1.13 & \cellcolor{c1!13}1.26 & \cellcolor{c1!15}1.28 \\
\rowcolor{tint1}\textbf{all replies} & \textbf{1.52} & \textbf{1.17} & \textbf{1.27} & \textbf{1.31} \\
\bottomrule
\end{tabular}
\end{table}

On LFM2.5-2.6B the scheduler's gain is largest where the reply repeats its input (Table~\ref{tab:sources}): 1.73$\times$ on LongBench's documents and 1.71$\times$ on ConvFinQA's reports, against 1.43$\times$ on WildChat's open conversation. Against the ablation's fixed 16-row tree (Table~\ref{tab:ladder}: this target's best width, without the fitted acceptance model and the n-gram candidates) the order is the same: 1.137$\times$ on LongBench, 1.086$\times$ on ConvFinQA and 1.020$\times$ on WildChat. On the routed target the gain varies less by source, because most of it comes from finding the target's width, which helps every source alike. On the Qwen3 targets it is largest on ConvFinQA and smallest on ToolACE.

\subsection{What the chooser picks}\label{sec:4.5}

\begin{figure*}[t]
\centering
\includegraphics[width=0.8\textwidth]{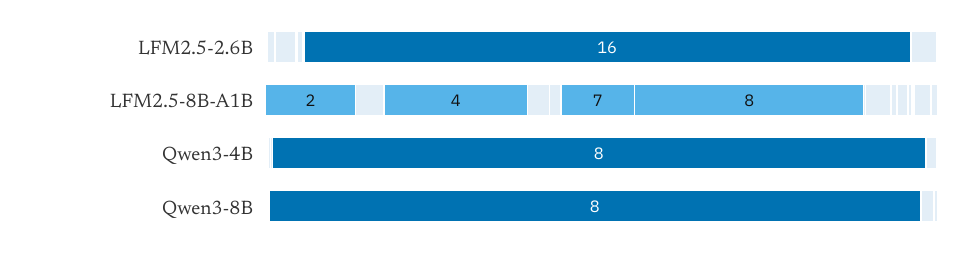}
\caption{Share of rounds at each verified width, per target, from AdaSpark's rounds on the instrumented run over the short subset (setup notes), cold-start rounds excluded. LFM2.5-2.6B: 7 rows 1\%, 8 rows 3\%, 16 rows 90\%, 32 rows 4\%, other widths under 1\% (1,610 rounds); LFM2.5-8B-A1B: 2 rows 13\%, 3 rows 4\%, 4 rows 21\%, 5 rows 3\%, 6 rows 2\%, 7 rows 11\%, 8 rows 34\%, 10 rows 4\%, 12 rows 2\%, 16 rows 2\%, 32 rows 1\%, other widths under 1\% (3,066 rounds); Qwen3-4B: 8 rows 97\%, 16 rows 2\%, other widths under 1\% (3,296 rounds); Qwen3-8B: 8 rows 97\%, 16 rows 2\%, other widths under 1\% (4,094 rounds).}
\label{fig:widths}
\end{figure*}

LFM2.5-2.6B ran 16 rows in 90\% of rounds and 32 rows in 4\% (Figure~\ref{fig:widths}), and the Qwen3 targets 8 rows in 97\%: the best pinned widths of Table~\ref{tab:widths}, reached from an empty store without a sweep. The wider verifies are rare and mostly follow long n-gram candidates (\S\ref{sec:4.3}). On the routed target, where the verify time rises with every row, the chooser ran 2 to 7 rows in 55\% of its rounds and 8 rows in 34\%, moving between them from round to round. The narrower rounds commit fewer tokens in less time, at about the same rate, so the chooser ties the best pinned width, 8 rows.

Table~\ref{tab:classes} compares the learned widths of \S\ref{sec:3.2} with the 8-bit matrix kernel's row classes (every 8 rows up to 64, plus a tile boundary at 47 rows on the 8-bit target) on the evaluation subset, with learned-width runs before and after the kernel-class run on one build. The two tie on every target, within 0.4\% over all replies (0.996--1.002$\times$) and within the intervals in every context band but one: on Qwen3-8B under 2,048 tokens the learned widths are 0.3\% faster (1.003$\times$, interval 1.001--1.006). The learned tables offer most widths from 2 to 16 rows; on the routed target their 2- to 7-row rounds, about half of all rounds, commit 8\% fewer tokens in 8\% less time than the kernel classes' 8 rows: the same rate. What learning buys is that the widths need no declaration from the backend: the kernel classes come from the 8-bit kernel and the 4-bit and routed targets borrow them, while the learned widths follow whatever kernels each target runs. Every other AdaSpark result uses the learned widths.

\begin{table*}[t]
\centering\small
\caption{Learned against kernel width classes on the evaluation subset: AdaSpark's decode speed with learned widths (\S\ref{sec:3.2}) over its speed with the widths taken from the 8-bit matrix kernel's row classes, geometric mean of per-reply ratios with its 95\% bootstrap interval over conversations, for all replies and by the context a reply was decoded at; above 1 the learned widths are faster. Both arms start from an empty store and ran on one build; the learned arm runs before the kernel-class arm and again after it, and its speed on a reply is the mean of the two runs. The last columns give the widths each arm's table offered at the end of its run (setup notes).}
\label{tab:classes}
\setlength{\tabcolsep}{4pt}
\begin{tabularx}{\textwidth}{lrrrXX}
\toprule
\rowcolor{head}\textbf{target} & \textbf{all replies} & \textbf{under 2,048} & \textbf{2,048 and over} & \textbf{learned widths} & \textbf{kernel-class widths} \\
\midrule
LFM2.5-2.6B & \makecell[r]{0.997$\times$\\ {\scriptsize [0.992, 1.001]}} & \makecell[r]{0.995$\times$\\ {\scriptsize [0.990, 1.001]}} & \makecell[r]{1.005$\times$\\ {\scriptsize [0.995, 1.012]}} & 2--4, 6, 8--10, 12--17, 20, 22, 24, 26, 28, 32, 34, 36, 40, 48, 64 & 8, 16, 24, 32, 40, 47, 48, 56, 64 \\
LFM2.5-8B-A1B & \makecell[r]{0.996$\times$\\ {\scriptsize [0.990, 1.003]}} & \makecell[r]{0.994$\times$\\ {\scriptsize [0.987, 1.002]}} & \makecell[r]{1.005$\times$\\ {\scriptsize [0.992, 1.019]}} & 2--17, 20, 24, 28, 32, 64 & 8, 16, 24, 32, 40, 48, 56, 64 \\
Qwen3-4B & \makecell[r]{0.998$\times$\\ {\scriptsize [0.986, 1.008]}} & \makecell[r]{1.003$\times$\\ {\scriptsize [0.997, 1.010]}} & \makecell[r]{0.985$\times$\\ {\scriptsize [0.947, 1.012]}} & 2--18, 20, 24, 28, 30, 32, 44, 57 & 8, 16, 24, 32, 40, 48, 56, 64 \\
Qwen3-8B & \makecell[r]{1.002$\times$\\ {\scriptsize [0.999, 1.006]}} & \makecell[r]{1.003$\times$\\ {\scriptsize [1.001, 1.006]}} & \makecell[r]{0.999$\times$\\ {\scriptsize [0.989, 1.010]}} & 2--17, 20, 22, 24, 26, 28, 30, 32, 57 & 8, 16, 24, 32, 40, 48, 56, 64 \\
\bottomrule
\end{tabularx}
\end{table*}

\subsection{Cost-model accuracy}\label{sec:4.6}

\begin{figure*}[t]
\centering
\includegraphics[width=0.8\textwidth]{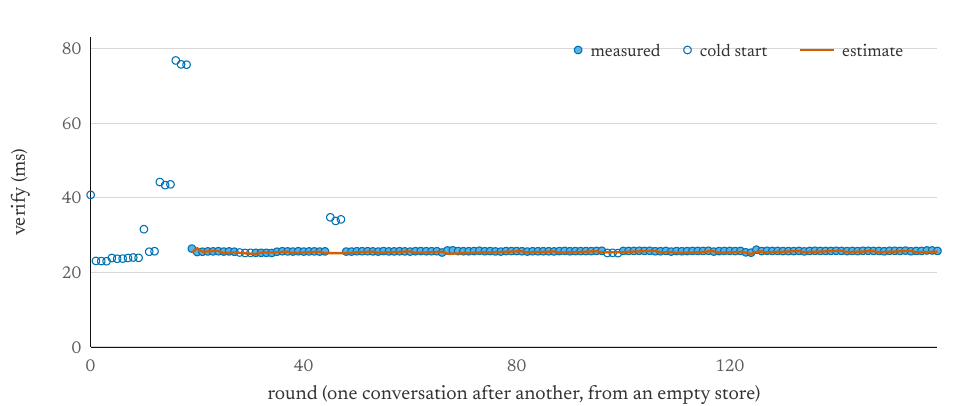}
\caption{LFM2.5-2.6B: measured verify time of each round (dots; hollow dots are cold-start rounds at a width chosen for measurement) and the cost model's estimate for the width it chose (line), from an empty store, over the first 160 rounds. Over all 1,677 rounds, 67 of them cold-start rounds, the median absolute error of the estimate over the second half of the run is 0.9\%.}
\label{fig:learning}
\end{figure*}

From an empty store (Figure~\ref{fig:learning}), the cost model on LFM2.5-2.6B spends 67 rounds of the first request measuring widths, 1.7~s of its 16.5~s of decode, and none in the rest of the run's 1,677 rounds; over the second half of the run its estimate for the chosen width is within 0.9\% of the measured verify time (median). On Qwen3-4B the first request spends 60 rounds (2.5~s of 31.8~s) and the median error is 1.6\%; on Qwen3-8B, 58 rounds (5.0~s of 60.5~s) and 0.8\%. On the routed target the first request spends 55 rounds (1.8~s of 20.3~s), but the median error is 6.9\%. The difference is in the verifies, not the fit: an 8-row routed verify varies with the experts its rows touch, and its median time is 6.5\% above its 15th percentile, against 1.2\% for a 16-row verify on the dense target. Part of that variation is real rather than interference, so the premise of Proposition~\ref{prop:3} holds only in part there, and the low quantile prices the routed verify below its typical time. The gap is similar at 6 rows (7.4\%), so 6 and 8 rows are still priced alike, and the chooser's widths there tie the best pinned width (Table~\ref{tab:widths}).

\subsection{Acceptance calibration and agreement with plain decode}\label{sec:4.7}

\begin{figure*}[t]
\centering
\includegraphics[width=0.8\textwidth]{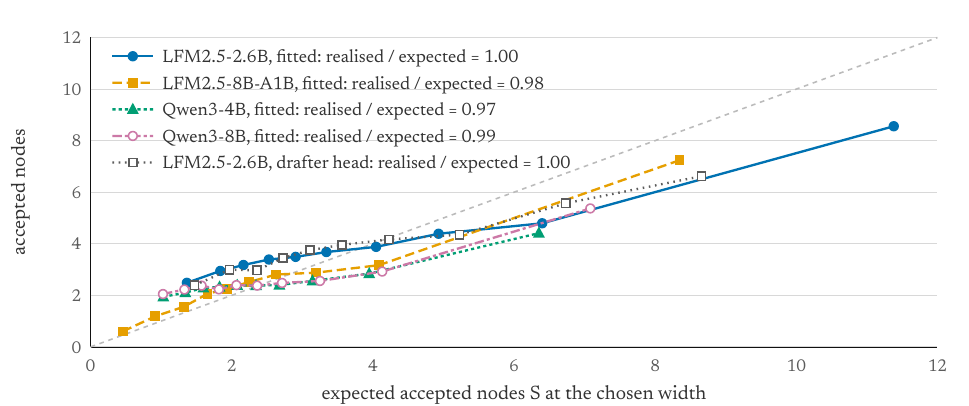}
\caption{Accepted nodes per round against the expected accepted nodes $S$ at the chosen width, in ten equal-count bins of $S$ per series, on the instrumented run over the short subset, over the rounds whose width the chooser set from $S$ (cold-start rounds have none). The dashed diagonal is perfect calibration. Realised over expected accepted nodes: LFM2.5-2.6B, fitted 0.997 (1,460 rounds); LFM2.5-8B-A1B, fitted 0.981 (1,998 rounds); Qwen3-4B, fitted 0.973 (3,208 rounds); Qwen3-8B, fitted 0.986 (3,972 rounds); LFM2.5-2.6B, drafter head 1.004 (1,642 rounds). The drafter-head series is the cost-model arm, which prices the tree with the head-based estimate.}
\label{fig:calibration}
\end{figure*}

Pooled over the rounds of Figure~\ref{fig:calibration}, the fitted model's expectation is within 2\% of what was accepted on both LFM2.5 targets and within 3\% on the Qwen3 targets, and the head-based estimate is within 1\% on LFM2.5-2.6B: in aggregate both are calibrated. Binned by $S$, every series lies above the diagonal at small $S$ and below it at large $S$: the models under-state what their least confident trees accept and over-state what their most confident ones accept. This compression is why the chooser rescales the expected gain of extra rows by a fitted factor (\S\ref{sec:3.4}) rather than taking $S$ at face value. The fitted model's advantage is per node, which is what the order and the width depend on: its log loss is 10.5--15.0\% lower than the head-based estimate's, and lower on every request after the first on all four targets (Table~\ref{tab:accept}, Figure~\ref{fig:accept-curve}).

\begin{table}[t]
\centering\small
\caption{Replies of the evaluation subset whose text, reasoning included, is byte-identical between AdaSpark's first run and imparo's three-pick DSpark. Later turns carry each arm's own earlier replies, so one early difference makes every later turn of that conversation differ; first turns share their history and isolate the verify.}
\label{tab:agree}
\setlength{\tabcolsep}{4pt}
\begin{tabular}{lrr}
\toprule
\rowcolor{head}\textbf{target} & \textbf{all replies} & \textbf{first turns} \\
\midrule
LFM2.5-2.6B & 81/95 & 39/42 \\
LFM2.5-8B-A1B & 73/95 & 34/42 \\
Qwen3-4B & 78/95 & 37/42 \\
Qwen3-8B & 85/95 & 38/42 \\
\bottomrule
\end{tabular}
\end{table}

On the evaluation subset, 81 of AdaSpark's 95 replies on LFM2.5-2.6B are byte-identical to imparo's three-pick DSpark, and 73 to 85 on the other targets (Table~\ref{tab:agree}). Every emitted token is the target's argmax on its verify row (\S\ref{sec:3.5}); we expect the differences to come from near ties in the rounds outside the width-invariant range of \S\ref{sec:3.5} (2-row verifies, wider verifies and cold-start measurements), which accumulate in another order. The routed target's choice of four experts per layer is a further point where a near tie can change the computation; we have not traced the individual divergences. Once a reply differs, the rest of its conversation differs too, since each arm carries its own replies as history.

\section{Discussion}\label{sec:5}

\paragraph{Why learn online.} The quantities a width decision needs vary in three ways. The verify's time depends on the target and its weight format (Figure~\ref{fig:verify-cost}): the same machine and kernel library give a nearly flat 16-row class on one target and a 50\% step at the 9th row on another. Each class's time depends on context, with its own slope (Table~\ref{tab:context-law}). And what a wider tree buys depends on the text being generated (Table~\ref{tab:sources}). A profile taken before serving can capture the first two, but only for the model file, quantisation, device and kernel build it was taken on, and must be retaken whenever one changes; the third changes while serving, with the workload. AdaSpark needs none of this in advance: its cost model starts from a ladder of six widths and its acceptance model from a prior derived from the features' definitions, and the only price is a cold start of 55 to 67 rounds in the first request, 8--10\% of its decode time.

\paragraph{Two learners, two clocks.} The two quantities change for different reasons, so AdaSpark fits them with two models. The verify's time is a property of the model file, the device, the kernels and the context: its law does not change while a model is served, so the cost model converges once per model and device, is stored, and afterwards tracks only the context and the machine's drift (\S\ref{sec:3.2}). Acceptance is a property of the text: it changes with every request and with the workload, so the acceptance model keeps learning, carries a per-request term that adapts within a reply (\S\ref{sec:3.3}), and must recover quickly after a switch. At the 20 changes of source in the evaluation run (five per target), the fitted model's advantage on the first request of a new source is typical of that source: it ranks 0.47 on average among the source's requests and falls below their median at 8 of the 20 (Figure~\ref{fig:accept-curve}). What differs most between workloads, the n-gram evidence and the context, enters the model as features rather than as something it must relearn, and no detector of workload changes was needed on these workloads. One learner would force one rate on both: either the cost would be relearned from noise every request, or acceptance would lag a workload change by as long as the cost takes to converge.

\paragraph{Engine and scheduler.} The mix of the two factors depends on the target (Figure~\ref{fig:e2e}). On the Qwen3 targets most of the gain over llama.cpp is in the engine's speculative path, since the two engines' autoregressive decode is at parity there. The scheduler's factor is largest on the dense LFM2.5 target (1.52$\times$), whose verify is nearly flat to 16 rows, so a wide tree and n-gram candidates cost little, and smallest on the routed target (1.17$\times$), where every row costs expert traffic. The Qwen3 targets sit between (1.27$\times$ and 1.31$\times$): their 4-bit verify steps at the 9th row, so the chooser stays at 8 rows.

\paragraph{Against the best pinned width.} A width tuned by a sweep exists only after the sweep, which must be repeated for every target and weight format; AdaSpark reaches it without one. Where one width is right in almost every round, as on the dense targets, a tie with the sweep is the expected result. On the routed target, moving between widths from round to round also ties the best single width rather than beating it. A wrong pinned width is expensive: 12 rows costs 20--23\% on the Qwen3 targets, and on the routed target 16 rows, the dense LFM2.5 target's best width, costs 10\% (Table~\ref{tab:widths}).

\section{Related work}\label{sec:6}

Table~\ref{tab:sizing} summarises how schedulers for speculative decoding size a verify: whether the size is chosen anew for each verify, and where the two quantities the choice needs come from.

\begin{table*}[t]
\centering\small
\caption{How schedulers size a verify. ``Per round'' means the size is chosen anew for each verify from the listed estimates. Entries are taken from each paper or, for the engines, from their source code (\S\ref{sec:6}).}
\label{tab:sizing}
\setlength{\tabcolsep}{4pt}
\begin{tabularx}{\textwidth}{lXXX}
\toprule
\rowcolor{head}\textbf{method} & \textbf{size chosen per round} & \textbf{verify cost from} & \textbf{acceptance from} \\
\midrule
\textbf{AdaSpark} & yes: tree width, from widths it learns, 2--64 rows & fitted online from served rounds, per learned width, as a function of context & fitted online per node, with terms for n-gram candidates \\
CAST \cite{hong2026inference} & yes: per step, by a cost-benefit test with thresholds set by hand per model & latency table over batch size, context and tokens fed, profiled per device before serving & drafter probabilities \\
EVICT \cite{pan2026making} & yes: tree prefix, 1--32 tokens & table profiled at startup per verify length, no context term & drafter probabilities \\
Bastion \cite{oh2026bastion} & yes: grows the tree until the estimated speedup falls & roofline model in width and context, calibrated offline, optional online scale factor & drafter probabilities \\
TreeSpark \cite{zhou2026treespark} & yes: stops at a price per verified token & latency ladder measured once per deployment & two-parameter map fitted offline to target outcomes \\
DSpark \cite{cheng2026dspark}, vLLM \cite{kwon2023efficient}, SGLang \cite{zheng2024sglang} & yes: a chain length per request under a batch budget & table profiled at engine start \cite{cheng2026dspark}; at startup at 8,192 tokens \cite{kwon2023efficient}; offline, at contexts under about 600 tokens \cite{zheng2024sglang} & confidence head, calibrated offline \cite{cheng2026dspark} \\
TurboSpec \cite{liu2024turbospec} & yes: chain length & linear model fitted offline & running acceptance rate, online \\
Cascade \cite{saxena2025utility} & yes: chain length, re-tested periodically & measured online in short trial phases, no model & measured together with the cost \\
EAGLE-2 \cite{li2024eagleb}, OPT-Tree \cite{wang2025opt}, TALON \cite{liu2026talon} & tree shape adapts; node budget fixed & not modelled & drafter probabilities or confidence \\
PCTree \cite{li2026chains}, Trees from Marginals \cite{oda2026trees} & no: size fixed, found by a sweep & not modelled & drafter probabilities \\
Graft \cite{shen2026draft} & no: 60 nodes; the split between drafted and retrieved nodes changes by confidence stage & not modelled & drafter probabilities; retrieved nodes not priced \\
llama.cpp DSpark \cite{gerganovndllama} & no: a chain of up to 3 tokens by default & not modelled & optional cut-off on the confidence head, off by default \\
\bottomrule
\end{tabularx}
\end{table*}

\paragraph{Tree size from a cost.} CAST \cite{hong2026inference} sizes a speculative tree from a latency table over batch size, context and tokens fed, measured per device before serving, with thresholds set by hand per model. Sequoia \cite{chen2024sequoia} also profiles the hardware to size trees, and CaDDTree \cite{zhang2026cost} sizes diffusion draft trees from a cost curve fitted offline once per model and device; DDTree \cite{ringel2026accelerating} builds draft trees from a block-diffusion drafter. Three recent schedulers choose the tree size each round from a cost and the prefix property of Proposition~\ref{prop:1}, two by the ratio of expected accepted tokens to cost and one by a price per verified token. EVICT \cite{pan2026making} does this for EAGLE-3 trees on mixture-of-experts targets, with acceptance from the draft probabilities and a step-latency table profiled at initialisation per verify length, with no context term. Bastion \cite{oh2026bastion} does it for block-diffusion drafters with a roofline model in width and context, calibrated offline per GPU and target and correctable online by one multiplicative factor; its stopping rule assumes a cost convex in the width. TreeSpark \cite{zhou2026treespark} builds DSpark trees from parent-conditioned candidates and stops at a price read from a latency ladder measured once per deployment. TurboSpec \cite{liu2024turbospec} and ASPIRE \cite{ziashahabi2026aspire} fit linear cost models offline, and D-cut \cite{liu2026d} prunes verify depth from a cost table profiled at startup. Cascade \cite{saxena2025utility} measures the utility of a few draft lengths online in short test phases, without a cost model, and AdaptiveSD \cite{saremi2026adaptivesd}, BanditSpec \cite{hou2025banditspec} and Sun et~al. \cite{sun2026delay} learn a chain's length online from measured outcomes. Other tree methods fix the budget: PCTree \cite{li2026chains} and Trees from Marginals \cite{oda2026trees} at a size found by a sweep, TALON \cite{liu2026talon} with a hand-set threshold, and ECHO \cite{hu2026echo} and Graft \cite{shen2026draft} with thresholds set in a warm-up run. EAGLE-2 \cite{li2024eagleb} and OPT-Tree \cite{wang2025opt} select a dynamic tree's nodes by the product of draft probabilities along the path, the rule of Proposition~\ref{prop:1}, under a fixed node budget. AdaSpark differs from all of these in fitting the verify cost while serving, for widths it learns and as a function of context, with no profile or calibration before serving, and in pricing time at the long-run rate over a cost with steps.

\paragraph{DSpark in serving engines.} vLLM \cite{kwon2023efficient} and SGLang \cite{zheng2024sglang} each ship a DSpark scheduler (Table~\ref{tab:vllm}). Both size a batch-wide verify budget by the ratio of expected accepted tokens to the cost of a step, and both take expected acceptance from the drafter's confidence head. vLLM's adaptive verification (\texttt{enable\_adaptive\_verification}, off by default) profiles its table at startup at one context length, 8,192 tokens by default, and makes it monotone by a running maximum; the table is a step function because CUDA-graph capture pads a batch to the next captured size. SGLang reads a table written by a separate offline profiler, which by default decodes 16-token prompts for about 570 steps, so its contexts stay below about 600 tokens; without that file it verifies every drafted token. Neither refits its table while serving.

\begin{table*}[t]
\centering\small
\caption{DSpark schedulers in vLLM (read from vLLM commit 41c4a3ed4e) and SGLang (read from SGLang commit 8b2ca8ecc2), and AdaSpark.}
\label{tab:vllm}
\setlength{\tabcolsep}{4pt}
\begin{tabularx}{\textwidth}{XXXX}
\toprule
\rowcolor{head}\textbf{} & \textbf{vLLM adaptive verification} & \textbf{SGLang DSpark scheduler} & \textbf{AdaSpark} \\
\midrule
objective & expected tokens / cost & expected tokens $\times$ steps per second & expected tokens $-$ $\rho$ $\cdot$ time (Eq.~7) \\
source of cost & profiled at startup & offline profiler, loaded from a file & fitted from served rounds (\S\ref{sec:3.2}) \\
context & one length, 8,192 tokens by default & 16-token prompts decoded for about 570 steps (contexts up to about 600 tokens) & per-class law in context, refitted while serving \\
where the table is not monotone & running maximum over narrower widths: a cheaper wide width is priced up and never used & lookup at the nearest profiled size below & minimum over wider classes: a cheaper wide class is used for the narrower widths \\
source of acceptance & drafter confidence head & drafter confidence head, optional offline calibration & fitted multinomial (\S\ref{sec:3.3}) \\
candidate sources per verify & one & one & two, priced separately (\S\ref{sec:3.3}) \\
\bottomrule
\end{tabularx}
\end{table*}

\paragraph{Predicting acceptance.} SpecDec++ \cite{huang2025specdec} trains an acceptance-prediction head onto the draft model and stops drafting when the predicted rejection probability crosses a threshold. DSpark \cite{cheng2026dspark} trains its confidence head on the per-position acceptance rate, calibrates it offline by temperature scaling, and passes the calibrated values to its scheduler. TreeSpark \cite{zhou2026treespark} reports that this head ranks tree edges poorly and replaces it with a two-parameter map from the drafter's probability, fitted once to the target's outcomes; TAPS \cite{wang2026taps} trains a scorer over each parent's children offline from recorded verifications; C2T \cite{huo2025c2t} trains a classifier on accepted and rejected tree nodes; and LiLiCorr \cite{rusanovsky2026lilicorr} trains a head on the drafter's confidence features. TurboSpec \cite{liu2024turbospec} and ASPIRE \cite{ziashahabi2026aspire} track an acceptance rate online as a running average, and BanditSpec \cite{hou2025banditspec} learns the draft length online as a bandit problem. D-PACE \cite{wu2026d} trains a parallel drafter with per-position loss weights derived from a surrogate of the expected accepted length. Our measurements (\S\ref{sec:3.3}, Table~\ref{tab:accept}) agree with TreeSpark's that a model fitted to the target's outcomes predicts acceptance better than the head; AdaSpark fits such a per-node model online, from an empty store, with separate terms for n-gram candidates and for agreement between the sources.

\paragraph{Text-derived candidates.} Prompt lookup decoding \cite{saxena2023prompt} matches n-grams against the prompt, REST \cite{he2024rest} retrieves continuations from a datastore, ANPD \cite{ou2024lossless} adapts the n-gram order online, the N-Grammys \cite{stewart2024n} batch several learning-free proposers, PLD+ \cite{somasundaram2025pld} and LogitSpec \cite{liu2026logitspec} refine which candidate to take, and SuffixDecoding \cite{oliaro2025suffixdecoding} and SAM Decoding \cite{hu2025sam} match suffixes of the prompt and of earlier outputs. These methods use the text as their only source or switch between a text source and a model drafter at each step: SuffixDecoding's hybrid mode and Arctic Inference \cite{snowflakendarctic} take the suffix draft when its score passes a threshold, TensorRT-LLM \cite{nvidiandtensorrt} lets a suffix-automaton draft override the neural one above a match length, SAM Decoding selects by match length, and \cite{eisenstadt2026copy} decides with a probe on the target's hidden states. DART \cite{liu2026dart} uses an n-gram model built offline to rerank a drafter's candidates without adding its own. Two methods we know of verify a model drafter's and retrieved candidates in one tree: RASD \cite{quan2025rasd} merges an EAGLE-2 tree with a retrieval tree, and Graft \cite{shen2026draft} prunes an EAGLE-3 tree at confidence checkpoints and fills the freed slots of a fixed 60-node budget from a token-successor table, following a fixed template per stage. Both divide the rows between the sources by fixed rules. In AdaSpark one fitted model prices both sources, so an n-gram candidate takes a row only from a candidate the model ranks below it, and widens the verify only when the cost model says its expected tokens pay for it. SpecInfer \cite{miao2024specinfer} merges the token trees of several small draft models into one verify. Among serving engines, llama.cpp uses the first speculator that returns a draft in each round, and the released versions of vLLM and SGLang construct one proposer per server.

\paragraph{Verify cost on mixture-of-experts targets.} MoESD \cite{huang2025moesd} models the verify cost of a mixture-of-experts target by the expected number of distinct experts the verified tokens activate. Cascade \cite{saxena2025utility} measures verify time growing two- to three-fold with draft length and chooses the length online by its utility, EVICT \cite{pan2026making} motivates its truncation by the same growth, MoE-Spec \cite{mcdanel2026moe} budgets the experts a verify loads, and \cite{amankulov2026limits} bounds the speculation budget offline with an oracle. Our routed target shows the same growth from 2 to 16 rows; we add its effect on the width choice beside dense targets on the same kernels, where the best width differs (\S\ref{sec:3.4}).

\paragraph{Other work on block drafters.} Work on the drafter itself is complementary to a scheduler: DeLS-Spec \cite{zheng2026dels} separates long and short context in the drafter, \cite{qiang2026beyond} analyses what a block drafter can and cannot predict at each position, Carryover Drafting \cite{koo2026carryover} reuses the states of rejected drafts, and \cite{liu2026when} runs the drafter's forward pass concurrently with the verify, which would shrink AdaSpark's fixed per-round cost $D$.

\section{Conclusion}\label{sec:7}

The time a tree verify takes is not a smooth function of its width. On the kernels we measured it steps at places that depend on the target and its weight format, each class's time grows with context at its own rate, and on a mixture-of-experts target the time rises with every row. What a wider tree buys depends on the drafter, the target and the text. A fixed width, or a width read from a table profiled before serving, fixes these quantities in advance.

AdaSpark fits both quantities from the rounds it serves: the widths worth offering and a low-quantile law in context for each one's verify time, and a multinomial over each parent's children for acceptance. The same model prices n-gram continuations of the request, so they join the drafter's tree and take rows only when they are worth them. Starting from an empty store, with constants set once on one target, it reaches the width a sweep finds best on each dense target, and it ties the best pinned width on the routed target, where it moves between widths from round to round. Against llama.cpp's DSpark with the same drafter files it decodes 1.52$\times$ to 3.10$\times$ faster over the whole test set; on the evaluation subset, with the engine held fixed, it is 1.17$\times$ to 1.52$\times$ faster than llama.cpp's default three-token chain. The questions it leaves open are a second device, sampling rather than greedy decoding, and verify batches shared between concurrent requests.

\bibliographystyle{unsrt}
\bibliography{refs}

\appendix

\section{Test set}\label{app:A}

Every prompt in the end-to-end evaluation comes from a published dataset; we wrote no prompt text. The protocol is ours: multi-turn conversations over each engine's HTTP API, with each engine's own replies carried as history. A method that learns from its workload is easy to favour with a workload chosen for it, so the text is fixed by the rule below and can be rebuilt by anyone.

\begin{table*}[t]
\centering\small
\caption{The test set: six public sources, 73 conversations, 148 turns, and the evaluation subset drawn from it. Licences as each source states them. The last column lists every change made to source text.}
\label{tab:testset}
\setlength{\tabcolsep}{4pt}
\begin{tabularx}{\textwidth}{XrrrrrX}
\toprule
\rowcolor{head}\textbf{source and content} & \textbf{licence} & \textbf{categories} & \textbf{conv.} & \textbf{turns} & \textbf{subset turns} & \textbf{change to the text} \\
\midrule
Spec-Bench \cite{xia2024unlocking}\newline MT-Bench two-turn items, translation, summarisation, QA, math reasoning, RAG & Apache-2.0 & 13 & 18 & 26 & 21 & none \\
SPEED-Bench \cite{abramovich2026speed}\newline qualitative split & source's own & 11 & 19 & 23 & 15 & none \\
WildChat \cite{zhao2024wildchat}\newline real user conversations; user turns only & ODC-BY & 1 & 8 & 32 & 16 & cut to the first four user turns \\
ConvFinQA \cite{chen2022convfinqa,xie2023pixiu}\newline several questions about one earnings report & not stated & 1 & 8 & 29 & 16 & later turns reduced to their own question \\
ToolACE \cite{liu2025toolace}\newline tool use with the source's schemas and recorded tool results & Apache-2.0 & 1 & 8 & 26 & 13 & Python-style parameter types (\texttt{dict}, \texttt{float}, \texttt{int}) renamed to the JSON Schema types they name \\
LongBench \cite{bai2024longbench}\newline long documents with the benchmark's own task prompt & MIT & 6 & 12 & 12 & 6 & a document over the cap cut from the middle \\
\bottomrule
\end{tabularx}
\par\smallskip\parbox{\linewidth}{\footnotesize The FLARE release of ConvFinQA repeats the report and the dialogue history in every turn; the harness carries the history itself, so later turns keep only their own question. LongBench's middle cut is the benchmark's own rule for inputs longer than a model's window.}
\end{table*}

\paragraph{Selection.} Items are taken in each dataset's own order: the first that pass one length cap (4,000 characters of prompt; 12,000 for ConvFinQA, ToolACE and LongBench, whose items carry a report, tool schemas or a document). A LongBench document over its cap is cut from the middle rather than skipped. The benchmark sets take a category's items until the category holds at least two turns, so a multi-turn item can take it to three. The single-category sets take their first eight conversations that pass the dataset's own filters (for WildChat: English, not flagged toxic, not redacted, at least four user turns). There is no sampling, and no item was examined for how any engine performs on it. SPEED-Bench's qualitative split reuses some Spec-Bench questions and repeats some of its own first turns; the rule keeps every copy it selects, and two conversations of the evaluation subset (humanities and QA) open with the same question as a Spec-Bench conversation. AdaSpark's n-gram index holds only the current request's text, so a repeated question gains nothing from an earlier copy. The conversations' text was fixed and committed before the first end-to-end run; a conversation's hash covers its turns, and ToolACE's tool schemas are converted by the rule in Table~\ref{tab:testset}. Each conversation records its source row identifier and a hash of its text, and each file records the source revision it was built from; rebuilding the set six days after the first build reproduced all 73 hashes.

\paragraph{Evaluation subset.} Replies with reasoning average about 2,000 tokens, so one pass of one arm over all 148 turns takes over an hour on the smallest target and many hours on the largest. The end-to-end runs therefore use a subset fixed by rule before any run: the first conversation of every category of a multi-category source, and the first half of the conversations of a single-category source. It keeps every category, 42 conversations and 87 of the 148 turns (Table~\ref{tab:testset}); a tool turn gives two timed replies, so the subset has 95. The other 31 conversations and 61 turns were then run with llama.cpp's DSpark and AdaSpark as a held-out check (\S\ref{sec:4.1}).

\paragraph{Departures from the sources' protocols.} Replies run to their end, up to 8,192 tokens, rather than to Spec-Bench's cap of 1,024, because a reasoning model's reply would be cut inside its reasoning at that length. Decoding is greedy, which this verifier requires. Under greedy decoding a reasoning model occasionally repeats itself until the cap; Qwen3's documentation warns of this for its reasoning mode. Such a loop is trivially predictable text that favours every speculative method, so a reply that reaches the cap is left out of every speed ratio involving an arm in which it did, and counted separately. In Qwen3-8B's three-token chain arm the machine slept during one conversation; its five replies were run again on the same binary, with byte-identical text, and the new timings replace the old. Reasoning is on for every target: LFM2.5's chat template has no switch for it, and Qwen3 is run with it on. A tool turn gives two timed replies, the call and the answer after the recorded result is returned; when the model answers without calling, the second reply does not run.

\section{Constants}\label{app:B}

Table~\ref{tab:constants} lists every constant the scheduler uses, with its value and how it was set. ``Hand'' means chosen once from what the quantity means, not by a search; ``dev'' means chosen by comparing a few values on the development corpus (setup notes). All were set while developing on LFM2.5-2.6B and are used unchanged on the four targets.

\begin{table*}[t]
\centering\small
\caption{The scheduler's constants.}
\label{tab:constants}
\setlength{\tabcolsep}{4pt}
\begin{tabularx}{\textwidth}{lXXl}
\toprule
\rowcolor{head}\textbf{part} & \textbf{constant} & \textbf{value} & \textbf{set by} \\
\midrule
cost model (\S\ref{sec:3.2}) & quantile level $\tau$ & 0.15 & hand (Proposition~\ref{prop:3}) \\
 & context basis unit & 8,192 tokens & hand \\
 & L1 and L2 weights $\lambda_{1}$, $\lambda_{2}$ (on time normalised by the class's scale) & 0.002, 0.01 & hand \\
 & FTRL rate and stabiliser & 0.05, 1 & hand \\
 & squared-gradient cap (rate floor) & 400 & hand \\
 & sample bank per context band; samples a refit is judged on & 16 samples; the newest quarter of a band, all of a band with fewer than four & hand \\
 & law clamp, relative to the class's lowest time & [0.5, 3] & hand \\
 & level window; level clamp & 8 verifies; [0.5, 2] & hand \\
 & measurements before a width is priced; samples before its law is fitted & 3; 3 & hand \\
 & starting ladder & 2, 4, 8, 16, 32, 64 rows & hand \\
 & rounds the chooser holds a width before a split; steep interval & 8; a rise above 25\% & hand \\
 & most widths learned & 24 & hand \\
width choice (\S\ref{sec:3.4}) & long-run rate half-life; rounds before a band's rate is used & 64 rounds; 8 & hand \\
 & gain clamp; shrinkage toward 1 & [0.25, 4]; ridge weight 0.02 times each regressor's sum of squares & hand \\
 & gain reference width $n_{0}$ & 8 rows & hand \\
 & ratio margin before the band's rate is known & 4\% & hand \\
 & hysteresis margin & 0.04 of the incumbent's expected tokens & dev \\
acceptance model (\S\ref{sec:3.3}) & drafter and depth/context terms: FTRL rate, stabiliser, L2 & 0.05, 1, 0.01 & hand \\
 & n-gram bucket FTRL rate & 0.5 & hand \\
 & gate on leaving the prior & $3\sqrt{m}$ after $m$ updates & hand \\
 & squared-gradient cap & 400 & hand \\
 & request intercept: AdaGrad rate, band, squared-gradient cap & 0.15, $\pm$2, 16 & hand \\
 & guard half-life & 256 labelled parents & hand \\
 & head-based estimate's per-request offset: AdaGrad rate, band & 0.15, $\pm$2 & hand \\
n-gram fill (\S\ref{sec:3.3}) & follower context when no earlier passage continues & last 3 tokens & hand \\
 & chain depth limit & 32 tokens below the anchor & hand \\
 & exploration period & every 16th tree round with n-gram candidates & dev \\
drafter (\S\ref{sec:2.2}) & candidates kept per position $K$ & 8 & hand \\
\bottomrule
\end{tabularx}
\end{table*}

\section{Alternatives we measured and rejected}\label{app:C}

Table~\ref{tab:rejected} lists the design alternatives we built and measured before settling on the form of \S\ref{sec:3}, what went wrong with each, and the measurement that showed it.

\begin{table*}[t]
\centering\small
\caption{Design alternatives that were built and measured before the form of \S\ref{sec:3}, all measured on LFM2.5-2.6B during development. Changes are in milliseconds per token; positive is slower.}
\label{tab:rejected}
\setlength{\tabcolsep}{4pt}
\begin{tabularx}{\textwidth}{XXX}
\toprule
\rowcolor{head}\textbf{alternative} & \textbf{what went wrong} & \textbf{measured} \\
\midrule
Per-round ratio without the fixed cost $D$ & narrow rounds looked cheap & +2.2\% against a fixed width of 16 (1,596-token context); $-$0.8\% with $D$ \\
Argmax of the value with no margin & the choice flipped on noise near the boundary & +1.70\% / $-$0.95\% against the best fixed width at two contexts; $-$0.40\% / $-$0.42\% with the margin \\
One acceptance curve per width, learned from rounds at that width & a round at $n$ rows says nothing about wider trees, so the learner never left its start & stayed at its initial width in 56 of 56 rounds from 8 rows and 47 of 47 from 16 \\
Drafter candidates priced without the contradicted term & picks next to a strong n-gram proposal were over-priced; the drafter's gain fell to 0.5 and trees stayed at 8 rows & contradicted picks: 0.598 predicted, 0.539 accepted; their siblings 0.041 and 0.012. With the term, picks contradicted by a copy of $ \geq 16$ tokens: 0.129 predicted, 0.147 accepted \\
Agreed candidates starting from the drafter's own estimate & agreement was not counted & agreed picks: 0.865 predicted, 0.970 accepted \\
Monotone cost table (running maximum over narrower widths) & a cheaper wide class can never be chosen & by construction; 48 rows are 4\% cheaper than 40 (\S\ref{sec:3.2}) \\
\bottomrule
\end{tabularx}
\end{table*}

\end{document}